\PassOptionsToPackage{table}{xcolor}
\documentclass{article}

\usepackage{microtype}
\usepackage{graphicx}
\usepackage{subcaption}
\usepackage{booktabs}
\usepackage{amsmath}
\usepackage{amssymb}
\usepackage{mathtools}
\usepackage{amsthm}
\usepackage{xspace}
\usepackage{tikz}
\usepackage{circuitikz}
\usepackage{bm}
\usepackage{enumitem}

\usepackage{hyperref}
\usepackage[capitalize,noabbrev]{cleveref}
\usepackage[textsize=tiny]{todonotes}

\usepackage[preprint]{icml2026}

\newcommand{\Comb}{\textsc{Comb}\xspace} %
\newcommand{\Agg}{\textsc{Agg}\xspace}
\newcommand{\Readout}{\textsc{Readout}\xspace}
\newcommand{\best}[1]{\textbf{\textcolor{red!80!black}{#1}}}
\newcommand{\secondbest}[1]{\textbf{\textcolor{blue!80!black}{#1}}}
\newcommand{\err}[1]{\scalebox{0.7}{\,$\pm$ #1}}
\newcommand{\tp}{\mathsf{T}}

\theoremstyle{plain}
\newtheorem{theorem}{Theorem}[section]
\newtheorem{proposition}[theorem]{Proposition}
\newtheorem{lemma}[theorem]{Lemma}
\newtheorem{corollary}[theorem]{Corollary}
\theoremstyle{definition}

\theoremstyle{remark}

\icmltitlerunning{XIMP: Cross Graph Inter-Message Passing}

\begin{document}

\twocolumn[
  \icmltitle{XIMP: Cross Graph Inter-Message Passing for Molecular Property Prediction}

  \icmlsetsymbol{equal}{*}

  \begin{icmlauthorlist}
    \icmlauthor{Anatol Ehrlich}{univie,docs}
    \icmlauthor{Lorenz Kummer}{univie,docs}
    \icmlauthor{Vojtech Voracek}{univie}
    \icmlauthor{Franka Bause}{univie}
    \icmlauthor{Nils M.~Kriege}{univie,ds}
  \end{icmlauthorlist}

  \icmlaffiliation{univie}{Faculty of Computer Science, University of Vienna, Vienna, Austria}
  \icmlaffiliation{docs}{Doctoral School Computer Science, University of Vienna, Vienna, Austria}
  \icmlaffiliation{ds}{Research Network Data Science, University of Vienna, Vienna, Austria}

  \icmlcorrespondingauthor{Nils M.~Kriege}{nils.kriege@univie.ac.at}

  \icmlkeywords{Graph Neural Networks, Hierarchical Graphs, Drug Discovery, Molecular Property Prediction, ADMET}

  \vskip 0.3in
]

\printAffiliationsAndNotice{}

\begin{abstract}
    Accurate molecular property prediction is central to drug discovery, yet graph neural networks often underperform in data-scarce regimes and fail to surpass traditional fingerprints. We introduce \emph{cross-graph inter-message passing} (XIMP), which performs message passing both \emph{within} and \emph{across} multiple related graph representations.
    For small molecules, we combine the molecular graph with scaffold-aware junction trees and pharmacophore-encoding extended reduced graphs, integrating complementary abstractions.
    While prior work is either limited to a single abstraction or non-iterative communication across graphs, XIMP supports an arbitrary number of abstractions and both direct and indirect communication between them in each layer. Across ten diverse molecular property prediction tasks, XIMP outperforms state-of-the-art baselines in most cases, leveraging interpretable abstractions as an inductive bias that guides learning toward established chemical concepts, enhancing generalization in low-data settings.
\end{abstract}

\section{Introduction} \label{sec:introduction}

In recent years, the field of graph representation learning has experienced substantial research interest with a strong focus on improving neural architectures. The graph model for representing real-world objects as graphs has received only limited attention, even though it often involves various design choices, including the level of abstraction. We argue that simultaneously using multiple alternative graph models can introduce a suitable inductive bias, thereby improving sample efficiency and predictive performance. We investigate this hypothesis in the domain of cheminformatics, where machine learning methods are widely used in the early stage of drug discovery~\cite{drugdiscovery}.

Classical methods in this domain rely on molecular graphs with atom and bond labels and encode specific substructures as bit vectors called \emph{chemical fingerprints}~\citep{rogers_extended-connectivity_2010, daylight_daylight_2008, durant_reoptimization_2002}. The predictive performance of fingerprints may vary across tasks, leading to the development of tailored fingerprints~\citep{durant_reoptimization_2002,erg}.
More recently, graph neural networks (GNNs) have been applied to molecular graphs to learn task-specific representations without relying on domain experts or feature engineering.~\citep{gnn1, gnn2, gnn3, wieder_compact_2020}. However, recent studies indicate that on key tasks, GNNs often fail to surpass standard fingerprint-based methods~\cite{stepisnik_comprehensive_2021,jiang_could_2021,dablander_exploring_2023}, with gains limited to some larger datasets~\cite{jiang_could_2021,deng_systematic_2023}.
Unfortunately, in many molecular property prediction tasks, high-quality data is scarce and may be insufficient to observe the methodological advantages of GNNs.

A common remedy to improve the sample efficiency of neural architectures is injecting domain knowledge into the model, e.g., via chemical structure augmentation~\citep{magar_auglichem_2022} or by incorporating representation-specific symmetries~\cite{atz_geometric_2021,cremer2023equivariant}. These inductive biases align models with chemical principles, boosting generalization in low-data regimes: equivariant models can enforce physical symmetries, while augmentations expose property-preserving variations. 
Such domain knowledge is implicitly encoded in different graph models for small molecules. In particular, so-called \emph{reduced graphs} provide a complementary bias by abstracting molecules into functional groups, rings, and pharmacophores, and relating them via edges~\cite{ft,erg,reducedgraphs}, thereby emphasizing features central to protein-ligand recognition and activity.
While reduced graphs are often derived from molecular graphs and thus do not introduce additional information, they emphasize chemically relevant patterns, enabling suitable architectures to learn more informative representations.
However, the integration of interpretable abstractions and their cross-communication within GNNs remains underexplored. In cheminformatics, multiple abstractions have been (i) processed sequentially~\citep{ergsimilar}, or, when simultaneous, either (ii) message passing is limited to local aggregation~\citep{functionalgroup}, (iii) communication occurs only indirectly via the molecular graph~\citep{li2024neural}, or (iv) only a single abstraction is used~\citep{himp,li2024neural,ergsimilar,wollschlager_expressivity_2024}.

\subsection{Our Contribution} \label{sec:contributions}
We propose \emph{cross graph inter-message passing} (XIMP), performing message passing iteratively within and across an arbitrary number of related graph representations. 
Specifically for cheminformatics, we combine the molecular graph with two complementary abstractions: the junction tree (JT)~\citep{pmlr-v80-jin18a}, which captures hierarchical fragment organization, and extended reduced graphs (ErG)~\citep{erg}, which encode pharmacophoric features and topology. We investigate different message-passing topologies, including direct communication between the abstractions and indirect communication via the central molecular graph.
Our method generalizes prior inter-message passing methods and can improve their expressivity.
Empirically, we compare XIMP against various GNN baselines and its closest competitors.
We show that XIMP outperforms these state-of-the-art models across many molecular property prediction tasks, indicating that it effectively leverages complementary graph abstractions. Moreover, XIMP surpasses state-of-the-art fixed representations based on ECFP~\cite{rogers_extended-connectivity_2010}, demonstrating the potential of feature learning for molecular property prediction.

\subsection{Related Work} \label{sec:bg}
Various standard GNNs can operate on molecular graphs~\cite{gnn2,gcn,gin,graphsage}. However, as the expressivity of approaches based on message passing is generally limited by the Weisfeiler-Leman (1-WL) test, they do not yield universal function approximators for graphs~\citep{express,wlsurvey}. They can even fail on tasks such as small-cycle detection~\citep{cycledetection2}, crucial for properties tied to local structures (e.g., aromatic rings). Remedies include invariant graph networks~\citep{ign1,ign2}, relational pooling~\citep{pooling,cycledetection2}, and higher-order WL extensions~\citep{higherorder}. While these techniques are often computationally demanding, a lightweight approach to improve expressivity is to combine graph models. Hierarchical inter-message passing (HIMP) combines the molecular graph with JTs, allowing, for example, to distinguish decalin and bicyclopentyl~\cite{himp}, which are indistinguishable by 1-WL based on their molecular graphs, thereby capturing subtle, biologically relevant variations.

However, integrating chemically interpretable reduced graphs and modeling their communication within message passing remains underexplored. JTs have been employed primarily for molecule generation~\citep{junction1,pmlr-v80-jin18a}. For property prediction, RG-MPNN~\citep{ergsimilar} applies an ErG-like reduction but processes graphs sequentially, unlike HIMP’s simultaneous scheme. Recent works employ multiple reductions, including JT, ErG, and functional-group graphs~\citep{functionalgroup}, but restrict message passing to each graph and pool via a super-node~\citep{allreduced,kengkanna2024enhancing}. \citet{wollschlager_expressivity_2024} introduce substructure-aware biases, though limited to a single fragment abstraction and neighbor-only communication. Their bias is mainly topological (rings, paths, junctions), lacking explicit aromaticity, pharmacophoric rules, or electronic effects. Thus, while multiple chemical abstractions have been used, prior work does not fully exploit inter-graph communication.

Generic graph abstractions have been proposed for graph pooling, including differentiable methods for learning hierarchical abstractions~\citep{ying2018hierarchical,topk}. More recently, bidirectional message passing between a graph and its abstraction has been proposed to improve the long-range performance.
Neural atoms~\citep{li2024neural} introduce fixed-size reduced graphs, introducing shortcuts for communication between distant atoms. For general graphs, \citet{finder2025improving} show that message passing between abstractions can mitigate oversquashing and improve long-range performance.

\section{Preliminaries} \label{sec:preliminaries}

This section presents graph abstractions for molecules, neural architectures for single and multiple graphs, and the notation used throughout this work.

\begin{figure*}
    \centering
    \begin{subfigure}[t]{\columnwidth}
        \centering
        \resizebox{\columnwidth}{!}{\input{figures/tikzfigure_mol_to_jt}}
        \caption{The junction tree decomposition converts a molecular structure (a1) into a molecular graph (a2), where nodes are atoms and edges are bonds. A cluster graph (a3) is built by grouping atoms in the same ring and the atoms of non-cyclic edges into clusters, connecting two clusters if they share atoms. Cycles in the cluster graph are removed by adding the shared atom as a separate cluster, yielding the junction tree (a4).}
        \label{fig:junction_tree}
    \end{subfigure}\hfill
    \begin{subfigure}[t]{\columnwidth}
        \centering
        \resizebox{\columnwidth}{!}{\begin{tikzpicture}
  
\pgfdeclarelayer{background}\pgfsetlayers{background,main}

\tikzstyle{every node}=[font=\LARGE]
\node [font=\bfseries, scale=3.5] at (49.0,21.0) {\textbf{(b2) Extended reduced Graph (ErG)}};

\node [font=\bfseries, scale=3.5] at (31.0, 21.) {\textbf{(b1) Molecular Graph}};

\clip (22,2.3) rectangle (59,22);  
 
\draw  (42.25,13) circle (0cm);
\draw [ fill={rgb,255:red,173; green,127; blue,168} , line width=2pt ] (43.25,11.25) circle (0.5cm) ;
\draw [ fill={rgb,255:red,173; green,127; blue,168} , line width=2pt ] (48.25,11.25) circle (0.5cm) ;
\draw [ fill={rgb,255:red,204; green,0; blue,0} , line width=2pt ] (50.75,10) circle (0.5cm) ;
\draw [ fill={rgb,255:red,204; green,0; blue,0} , line width=2pt ] (50.75,12.5) circle (0.5cm) ;
\draw [ fill={rgb,255:red,237; green,212; blue,0} , line width=2pt ] (53.25,12.5) circle (0.5cm) ;
\draw [ fill={rgb,255:red,114; green,159; blue,207} , line width=2pt ] (55.75,13.75) circle (0.5cm) ;
\draw [ fill={rgb,255:red,114; green,159; blue,207} , line width=2pt ] (55.75,11.25) circle (0.5cm) ;
\draw [ fill={rgb,255:red,114; green,159; blue,207} , line width=2pt ] (58.25,11.25) circle (0.5cm) ;
\draw [line width=2.75pt, short] (51.25,12.5) -- (52.75,12.5);
\draw [line width=2.75pt, short] (56.25,11.25) -- (57.75,11.25);
\draw [line width=2.75pt, short] (48.75,11.5) -- (50.25,12.25);
\draw [line width=2.75pt, short] (53.75,12.75) -- (55.25,13.5);
\draw [line width=2.75pt, short] (48.75,11) -- (50.25,10.25);
\draw [line width=2.75pt, short] (53.75,12.25) -- (55.25,11.5);
\draw [ fill={rgb,255:red,193; green,125; blue,17} , line width=2pt ] (45.75,12.5) circle (0.5cm) ;
\draw [ fill={rgb,255:red,193; green,125; blue,17} , line width=2pt ] (45.75,10) circle (0.5cm) ;
\draw [ fill={rgb,255:red,114; green,159; blue,207} , line width=2pt ] (53.25,8.75) circle (0.5cm) ;
\draw [line width=2.75pt, short] (50.75,10.5) -- (50.75,12);
\draw [line width=2.75pt, short] (46.25,10) -- (50.25,10);
\draw [line width=2.75pt, short] (43.75,11) -- (45.25,10.25);
\draw [line width=2.75pt, short] (46.25,10.25) -- (47.75,11);
\draw [line width=2.75pt, short] (43.75,11.5) -- (45.25,12.25);
\draw [line width=2.75pt, short] (46.25,12.25) -- (47.75,11.5);
\draw [line width=2.75pt, short] (51.25,9.75) -- (52.75,9);
\draw  (29.5,14) circle (0cm);
\draw [ fill={rgb,255:red,136; green,138; blue,133} , line width=2pt ] (32.5,8.75) circle (0.5cm) ;
\draw [ fill={rgb,255:red,136; green,138; blue,133} , line width=2pt ] (35.25,12.25) circle (0.5cm) ;
\draw [ fill={rgb,255:red,136; green,138; blue,133} , line width=2pt ] (37.75,13.5) circle (0.5cm) ;
\draw [ fill={rgb,255:red,136; green,138; blue,133} , line width=2pt ] (37.75,11) circle (0.5cm);
\draw [ fill={rgb,255:red,136; green,138; blue,133} , line width=2pt ] (40.25,11) circle (0.5cm);
\draw [line width=2.75pt, short] (33.25,12.25) -- (34.75,12.25);
\draw [line width=2.75pt, short] (38.25,11) -- (39.75,11);
\draw [line width=2.75pt, short] (35.75,12.5) -- (37.25,13.25);
\draw [line width=2.75pt, short] (30.75,10) -- (32,9);
\draw [line width=2.75pt, short] (35.75,12) -- (37.25,11.25);
\draw [ fill={rgb,255:red,136; green,138; blue,133} , line width=2pt ] (23,13.5) circle (0.5cm) ;
\draw [ fill={rgb,255:red,136; green,138; blue,133} , line width=2pt ] (23,11) circle (0.5cm) ;
\draw [ fill={rgb,255:red,136; green,138; blue,133} , line width=2pt ] (25.5,9.75) circle (0.5cm) ;
\draw [ fill={rgb,255:red,136; green,138; blue,133} , line width=2pt ] (25.5,14.75) circle (0.5cm) ;
\draw [ fill={rgb,255:red,136; green,138; blue,133} , line width=2pt ] (28,11) circle (0.5cm) ;
\node [font=\large] at (28.5,10.5) {};
\draw [ fill={rgb,255:red,136; green,138; blue,133} , line width=2pt ] (28,13.5) circle (0.5cm) ;
\draw [ fill={rgb,255:red,136; green,138; blue,133} , line width=2pt ] (30.5,14) circle (0.5cm) ;
\draw [ fill={rgb,255:red,136; green,138; blue,133} , line width=2pt ] (30.5,10.5) circle (0.5cm) ;
\draw [ fill={rgb,255:red,136; green,138; blue,133} , line width=2pt ] (32.75,12.25) circle (0.5cm) ;
\draw [line width=2.75pt, short] (23,11.5) -- (23,13);
\draw [line width=2.75pt, short] (28,11.5) -- (28,13);
\node [font=\large] at (23.5,11.75) {};
\node [font=\large] at (51.25,10.75) {};
\node [font=\large] at (28.5,11.75) {};
\draw [line width=2.75pt, short] (23.5,13.75) -- (25,14.5);
\draw [line width=2.75pt, short] (26,10) -- (27.5,10.75);
\draw [line width=2.75pt, short] (25,10) -- (23.5,10.75);
\draw [line width=2.75pt, short] (27.5,13.75) -- (26,14.5);
\draw [line width=2.75pt, short] (32.25,12.5) -- (31,13.75);
\node [font=\large] at (32.25,12.75) {};
\draw [line width=2.75pt, short] (30,10.5) -- (28.5,11);
\draw [line width=2.75pt, short] (30,14) -- (28.5,13.5);
\draw [line width=2.75pt, short] (32.25,12) -- (31,10.75);
\node [font=\large, color={rgb,255:red,78; green,154; blue,6}] at (49,14.5) {};
\node [font=\large, color={rgb,255:red,32; green,74; blue,135}] at (50.25,14.25) {};
\draw [ color={rgb,255:red,204; green,0; blue,0}, line width=2.76pt, ->, >=Stealth, dashed] (23,10.5) .. controls (24.5,7.75) and (37.5,3) .. (43,10.75);
\draw [ color={rgb,255:red,204; green,0; blue,0}, line width=2.76pt, ->, >=Stealth, dashed] (23,14) .. controls (26.75,20) and (40.75,17.25) .. (43.25,11.75);
\draw [ color={rgb,255:red,204; green,0; blue,0}, line width=2.76pt, ->, >=Stealth, dashed] (25.5,9.25) .. controls (27.25,6.75) and (40.25,5.75) .. (43,10.75);
\draw [ color={rgb,255:red,204; green,0; blue,0}, line width=2.76pt, ->, >=Stealth, dashed] (28,10.5) .. controls (29.5,8) and (39.5,5.25) .. (43,10.75);
\draw [ color={rgb,255:red,204; green,0; blue,0}, line width=2.76pt, ->, >=Stealth, dashed] (25.5,15.25) .. controls (27,18) and (40.5,17.5) .. (43.25,11.75);
\draw [ color={rgb,255:red,204; green,0; blue,0}, line width=2.76pt, ->, >=Stealth, dashed] (28,14) .. controls (28.75,17.25) and (40.75,17) .. (43.25,11.75);
\draw [ color={rgb,255:red,32; green,74; blue,135}, line width=2.76pt, ->, >=Stealth, dashed] (28,10.5) .. controls (31,4.5) and (42.5,5) .. (45.5,9.5);
\draw [ color={rgb,255:red,32; green,74; blue,135}, line width=2.76pt, ->, >=Stealth, dashed] (28,14) .. controls (30,20) and (42.25,18.25) .. (45.5,13);
\draw [ color={rgb,255:red,78; green,154; blue,6}, line width=2.76pt, ->, >=Stealth, dashed] (30.5,10) .. controls (31,4.5) and (42,1.25) .. (48,10.75);
\draw [ color={rgb,255:red,78; green,154; blue,6}, line width=2.76pt, ->, >=Stealth, dashed] (30.5,14.5) .. controls (33.5,21.75) and (43.5,20.25) .. (48.25,11.75);
\draw [ color={rgb,255:red,78; green,154; blue,6}, line width=2.76pt, ->, >=Stealth, dashed] (28,14) .. controls (30.25,21) and (44.5,19.5) .. (48.25,11.75);
\draw [ color={rgb,255:red,78; green,154; blue,6}, line width=2.76pt, ->, >=Stealth, dashed] (28,10.5) .. controls (29.25,4.5) and (43.75,3.25) .. (48,10.75);
\node [font=\LARGE, color={rgb,255:red,78; green,154; blue,6}] at (28.75,8) {};
\draw [ color={rgb,255:red,78; green,154; blue,6}, line width=2.76pt, ->, >=Stealth, dashed] (32.75,12.75) .. controls (34.75,18.25) and (42.75,21.5) .. (48.25,11.75);
\draw [ color={rgb,255:red,32; green,74; blue,135}, line width=2.76pt, ->, >=Stealth, dashed] (31,10.5) .. controls (37.75,10) and (45.75,3) .. (50.5,9.5);
\draw [ color={rgb,255:red,32; green,74; blue,135}, line width=2.76pt, ->, >=Stealth, dashed] (32.75,12.75) .. controls (37.25,22.25) and (45.25,20) .. (50.5,13);
\draw [ color={rgb,255:red,32; green,74; blue,135}, line width=2.76pt, ->, >=Stealth, dashed] (32.5,8.25) .. controls (33.25,2.75) and (48,2.5) .. (53.25,8.25);
\draw [ color={rgb,255:red,32; green,74; blue,135}, line width=2.76pt, ->, >=Stealth, dashed] (35.25,12.75) .. controls (41.75,21.5) and (49.25,18.25) .. (53,13);
\draw [ color={rgb,255:red,32; green,74; blue,135}, line width=2.76pt, ->, >=Stealth, dashed] (37.75,14) .. controls (46.25,22) and (49.75,18.75) .. (55.25,14);
\draw [ color={rgb,255:red,32; green,74; blue,135}, line width=2.76pt, ->, >=Stealth, dashed] (37.75,10.5) .. controls (38.25,4.5) and (53.75,4.75) .. (55.75,10.75);
\draw [ color={rgb,255:red,32; green,74; blue,135}, line width=2.76pt, ->, >=Stealth, dashed] (40.25,11.5) .. controls (44.75,19.75) and (54,20.5) .. (58.25,11.75);

\def\bgpad{10.5mm}
\begin{pgfonlayer}{background}
  \filldraw[
    rounded corners=14pt,
    fill=white,
    draw=white,
    line width=0.8pt
  ]
    ($(current bounding box.south west)+(-\bgpad,-\bgpad)$)
    rectangle
    ($(current bounding box.north east)+(\bgpad,\bgpad)$);
\end{pgfonlayer}

\end{tikzpicture}}
        \caption{Molecular graph (b1) and extended reduced graph (b2) constructed in four steps: (1) adjust atom charges to reflect physiological conditions; (2) assign H-bond donor/acceptor properties; (3) identify and tag endcap groups (lateral hydrophobic features, including thioethers); (4) add a ring centroid (\emph{aromatic}/\emph{hydrophobic}), link it to substituted atoms/bridgeheads, remove unsubstituted atoms, and keep bonds among the rest.}
        \label{fig:erg}
    \end{subfigure}
    \caption{Molecular graph abstractions: (\subref{fig:junction_tree}) junction tree and (\subref{fig:erg}) extended reduced graph. Arrows indicate node mappings between the graphs; their colors encode singleton or group memberships.}
    \label{fig:mol_abstractions}
\end{figure*}

\subsection{Molecular Graphs and Reduced Graphs} \label{sec:molecular_graphs}
A molecular graph encodes the structure of a molecule, representing atoms as nodes and bonds as edges, see Figure~\ref{fig:mol_abstractions}a2. Graph models representing groups of atoms by single nodes and the relation between groups by edges are referred to as \emph{reduced graphs}~\citep{reducedgraphs}. While often extracted from the molecular graph, a reduced graph exposes key features and topological structure more explicitly, enabling more generalized analyses of the represented molecule. XIMP operates on two complementary reduced graphs that provide interpretable abstractions matching our chemistry setting. (i) The \emph{junction tree} (JT)~\citep{pmlr-v80-jin18a}, also used in HIMP, decomposes molecules into overlapping fragments such as rings and bridges, such that their overlaps define a tree structure, see Figure~\ref{fig:mol_abstractions}a4. Similar tree-based abstractions are widely used in cheminformatics~\citep{ft}. (ii) The \emph{extended reduced graph} (ErG)~\citep{erg} encodes pharmacophoric features and atom-level topological relationships relevant to biological activity, see Figure~\ref{fig:erg}. JTs and ErGs provide abstract molecular representations at different granularities. JTs condense the molecular graph into a minimal set of fragment nodes and simplify their relationships by organizing them into a tree structure. ErGs, on the other hand, are more fine-grained: they treat rings with explicit bridgehead handling and add node features encoding higher-level pharmacophoric information absent in JTs and standard molecular graphs. Together with the molecular graph, these views offer complementary, non-redundant perspectives on molecular topology. A detailed description is available in Appendix~\ref{sec:reduced_detail}.

\subsection{Message Passing Graph Neural Networks} \label{sec:gnn}

GNNs learn node representations using message passing (MP) to propagate information across nodes and edges~\citep{gnn2}. 
In molecular property prediction, GNNs operate on a molecular graph $G = (V, E)$, where $V$ is a set of atoms and $E$ is a set of bonds of a molecule. Each node $v$ in $V(G)$, and edge $(u, v)$ in $E(G)$ is endowed with an initial real-valued feature vector $\bm{x}^{(0)}_u$ and $\bm{e}_{(u, v)}$ representing, for example, the atom and bond type, respectively. The node embeddings are then updated via message-passing layers by aggregating their neighbors' embeddings. The $l$-th layer recursively updates the embedding $\bm{x}_v^{(l)}$ of node $v$ via
\begin{align*}
    \bm{m}_v^{(l)} &= \Agg_{\bm{\theta}_1^{(l)}}\!\left(
    \left\{\!\!\!\left\{
    \left( \bm{x}_w^{(l-1)},\, \bm{e}_{(w,v)} \right)
    \mid w \in \mathcal{N}(v)
    \right\}\!\!\!\right\}
    \right) \\
    \bm{x}_v^{(l)} &= \Comb_{\bm{\theta}_2^{(l)}}\!\left(
    \bm{x}_v^{(l-1)},\, \bm{m}_v^{(l)}
    \right),
\end{align*}
where $\{\!\!\{ \cdot \}\!\!\}$ denotes a multiset and $\mathcal{N}(v) = \{u \in V(G) \mid (u,v) \in E(G)\}$ the neighbors of the node $v$. The functions \Agg and \Comb are parameterized by $\bm{\theta}_1^{(l)}$ and $\bm{\theta}_2^{(l)}$, respectively, which are optimized during training. Ultimately, after passing through $L$ layers, the resulting node embeddings are combined via
\begin{equation*}
    \bm{h}_G = \Readout\left( \left\{\!\!\!\left\{ \bm{x}_v^{(L)} \mid v \in V(G) \right\}\!\!\!\right\} \right).
\end{equation*}
By training the model parameters, the resulting graph embedding is optimized for predicting the desired molecular property.

\subsection{Hierarchical Inter-Message Passing} \label{sec:himp}
HIMP~\citep{himp} is a GNN architecture that leverages two molecular graph models and suitable GNN models. One model operates on the molecular graph with edge labels (GIN-E), while the other operates on its corresponding JT without edge labels (GIN). Both models use message passing as described in Section~\ref{sec:gnn}. Additionally, an \emph{inter-message passing} (IMP) model facilitates information exchange between the two graph representations in each layer, as shown in Figure~\ref{fig:com_flow}. For a complete outline of HIMP, see Appendix~\ref{apx:himp}.

\begin{figure*}[t]
\centering
\resizebox{0.98\textwidth}{!}{%
  \input{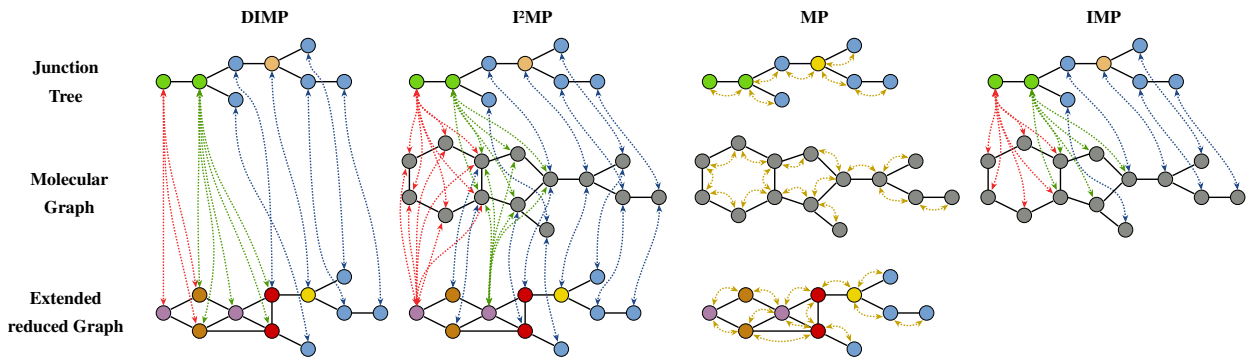}
}%
\caption{Visualization of communication flows in XIMP and HIMP. XIMP employs DIMP, I\textsuperscript{2}MP, and MP, while HIMP uses only IMP and MP. Node colors denote graph semantics; arrows indicate bidirectional message passing: red/green for one-to-many ring–abstraction relations, blue for one-to-one cross-graph relations, and yellow for one-to-one within-graph relations.}
\label{fig:com_flow}
\end{figure*}

\section{Cross Graph Inter-Message Passing} \label{sec:ximp}
XIMP performs message passing within and across arbitrary numbers of related graph representations, generalizing existing inter-message passing techniques. We consider a setting with a distinguished graph and multiple abstractions and investigate direct and indirect communication flows between them.
While our approach is domain-agnostic and not technically limited to hierarchical representations or molecular abstractions, we describe and illustrate it in the context of a central molecular graph and two or more complementary reduced graphs.

\subsection{Input and Formalism}
We denote the molecular graph by $G$ and its $n$ reduced graphs by $T_1, T_2, \dots, T_n$. All graphs have initial node embeddings, denoted by $\bm{X}^{(0)} \in \mathbb{R}^{|V(G)|\times d}$ for the molecular graph $G$ and $\bm{T}_i^{(0)} \in \mathbb{R}^{|V(T_i)|\times d}$ for the reduced graphs $T_i$, $i \in[1,n]$, respectively. We assume all node embeddings to be of dimension $d$, and, if necessary, apply a learnable transformation to project the given node attributes into a $d$-dimensional space.
A node in a reduced graph typically corresponds to multiple nodes in the molecular graph (see Figure~\ref{fig:mol_abstractions}). Formally, the $i$-th reduced graph $T_i$ is linked to the molecular graph via a node correspondence matrix $\bm{S}_i$ in $\{0,1\}^{|V(G)| \times |V(T_i)|}$. A value of 1 at position $(u,v)$ indicates that the $u$-th node of $G$ is represented by the $v$-th node of $T_i$. Hence, for graph abstractions, each row in $\bm{S}_i$ typically contains multiple $1$ entries. This formulation naturally accounts for overlapping abstractions, such as fused rings or functional groups, in which a node of $G$ relates to multiple reduced nodes. These matrices are used to pass messages among corresponding nodes in different graph representations.
The $l$-th layer of XIMP computes the node embeddings $\bm{X}^{(l)}$ and $\bm{T}_i^{(l)}$ from the node embeddings $\bm{X}^{(l-1)}$ and $\bm{T}_i^{(l-1)}$, for $i\in[1,n]$, respectively. We discuss the inter-message passing in the following and describe the overall framework in Section~\ref{sec:framework}.

\subsection{Indirect Inter-Message Passing (I\textsuperscript{2}MP)}\label{sec:i2imp}
Information exchange between the molecular graph and its reduced graphs occurs within each layer. Stacking multiple layers enables indirect communication between the reduced graphs (see Figure~\ref{fig:com_flow}).
The inter-message passing step computes intermediate embedding matrices for $i \in [1,n]$ according to
\begin{align*}
    \bm{X}_i^{(l)} &= \sigma \bigg( \widetilde{\bm{S}_i} \bm{T}_i^{(l-1)} \bm{W}_{i, 1}^{(l)} \bigg) 
\end{align*}
\begin{align*}
    \bm{M}_i^{(l)} &= \sigma \bigg( \widetilde{\bm{S}_i^\tp} \bm{X}^{(l-1)} \bm{W}_{i, 2}^{(l)} \bigg),
\end{align*}
where $\bm{W}_{i, 1}^{(l)}, \bm{W}_{i, 2}^{(l)} \in \mathbb{R}^{d \times d}$ are trainable parameters within layer $l$, and $\sigma$ denotes a non-linearity. 
The matrix $\widetilde{\bm{S}_i}=\bm{D}^{-1}\bm{S}_i$ with $\bm{D}=\mathrm{diag}(\bm{S}_i\bm{1})$ is the row-normalized version of the node correspondence matrix linking $G$ and $T_i$. Similarly, $\widetilde{\bm{S}^\tp_i}=(\bm{S}_i\bm{D}^{-1})^\tp$ is the row normalized transposed of $\bm{S}_i$. Note that the molecular graph $G$ plays a distinguished role, and we assume a given correspondence between its nodes and the nodes of all the reduced graphs $T_i$.

\subsection{Direct Inter-Message Passing (DIMP)} \label{sec:dimp}
Beyond exchanging information between the original molecular graph $G$ and each reduced graph $T_i$, we also allow direct communication between pairs of reduced graphs. We assume that each matrix $\bm{S}_i$ encodes a left-total relation, which ensures that every node of $G$ belongs to at least one node in each reduced graph $T_i$. For every pair $(i,k)$ with $i \neq k$, messages from $T_k$ to $T_i$ are computed by
\begin{align}
    \bm{M}_{k\rightarrow i}^{(l)} &= \sigma \left( \widetilde{\bm{S}}_{ik} \bm{T}_k^{(l-1)} \bm{W}_{k\rightarrow i}^{(l)} \right) \label{eq:dimp:msg}\\
    \widetilde{\bm{S}}_{ik} &= \bm{D}_{T,i}^{-1} \bm{S}_i^{\tp} \bm{D}_{G,k}^{-1} \bm{S}_k, \label{eq:dimp:norm}
\end{align}
where $\bm{W}_{k\rightarrow i}^{(l)} \in \mathbb{R}^{d\times d}$ are trainable parameters, $\sigma$ is a non-linearity,  $\bm{D}_{T,i}=\mathrm{diag}(\bm{S}_i^{\tp}\bm{1})$ stores for each node of $T_i$ the number of corresponding nodes in $G$, and $\bm{D}_{G,k}=\mathrm{diag}(\bm{S}_k\bm{1})$ gives for all nodes of $G$ the number of related nodes in $T_k$.
The left-normalizing matrix $\bm{D}_{G,k}^{-1}$ ensures that a node of $G$ with multiple memberships in $T_k$ does not contribute disproportionately when projecting from $T_k$ to $G$, while the right normalizing matrix $\bm{D}_{T,i}^{-1}$ balances the aggregation of multiple nodes of $G$ into a node of $T_i$. This prevents overweighting due to multiplicity, ensures stable message magnitudes across overlaps, and reflects the chemical intuition that atoms participating in several functional roles (e.g., ring junctions or substituents) should distribute their influence fairly among all abstractions. We formalize this by bounding the maximum absolute row and column sum of the matrix  $\widetilde{\bm{S}}_{ik} \bm{T}_k^{(l)}$ in Eq.~\eqref{eq:dimp:norm} before applying the linear transformation and activation (proof in Appendix~\ref{apx:dimpnorm}).

\begin{proposition} \label{prop:dimpnorm}
    Let $\bm{S}_k \in \{0,1\}^{|V(G)|\times |V(T_k)|}$ and $\bm{S}_i \in \{0,1\}^{|V(G)|\times |V(T_i)|}$ represent left-total relations between the nodes of $G$ and those of $T_k$ and $T_i$, respectively. Let $\widetilde{\bm{S}}_{ik}$ be defined according to Eq.~\eqref{eq:dimp:norm} and $\widetilde{\bm{M}}_{k\rightarrow i}^{(l)} = \widetilde{\bm{S}}_{ik} \bm{T}_k^{(l)}$.
    Then, the following statements hold:
    \begin{enumerate}[noitemsep, topsep=0pt]
        \item There exists $\alpha \in \mathbb{R}$ with $\frac{|V(T_i)|}{|V(T_k)|} \leq \alpha \leq |V(T_i)|$ such that
        $\|\widetilde{\bm{M}}_{k\rightarrow i}^{(l)}\|_\infty \leq \|\bm{T}^{(l)}_k\|_\infty$ and $\|\widetilde{\bm{M}}_{k\rightarrow i}^{(l)}\|_1 \leq \alpha \|\bm{T}^{(l)}_k\|_1$.
        \item For any $\bm{x} \in \mathbb{R}^{d}$ and $\bm{T}_k = \bm{1}\bm{x}^\tp \in \mathbb{R}^{|V(T_k)|\times d}$, it holds $\widetilde{\bm{S}}_{ik}\bm{T}_k = \bm{1}\bm{x}^\tp$.
    \end{enumerate}
\end{proposition}
Note that if $\bm{S}_k$ partitions the nodes of $G$, then $\bm{D}_{G,k} = \bm{I}$ and the expression reduces to the simpler form $\widetilde{\bm{S}}_{ik} = \bm{D}_{T,i}^{-1}\bm{S}_i^{\tp}\bm{S}_k$.

\subsection{Overall Framework}\label{sec:framework}
The $l$-th layer of XIMP computes the node embeddings $\bm{X}^{(l)}$ and $\bm{T}_i^{(l)}$ by combining intermediate embedding matrices defined in Sections~\ref{sec:i2imp} and~\ref{sec:dimp} with the embeddings from standard message passing as described in Section~\ref{sec:gnn}.
For each graph representation, we apply a suitable GNN layer and denote the computed node embeddings of the molecular graph $G$ and its reduced graphs $T_i$ by $\bm{\mathfrak{X}}^{(l)}$ and $\bm{\mathfrak{T}}^{(l)}_i$, for all $i \in [1,n]$, respectively.
Using the intermediate embedding matrices from inter-message passing, we compute the new node embeddings as output of the $l$-th layer according to
\begin{align}
  \bm{X}^{(l)} &= \bm{\mathfrak{X}}^{(l)} + \sum_{i=1}^n \bm{X}^{(l)}_i \label{eq:update1}\\
  \bm{T}_i^{(l)} &= \bm{\mathfrak{T}}_i^{(l)} + \bm{M}^{l}_i + \sum_{\substack{k \in [1,n] \\ k \neq i}} \bm{M}_{k\rightarrow i}^{(l)}. \label{eq:update2}
\end{align}
Note that all terms in Eqs.~\eqref{eq:update1},~\eqref{eq:update2} are computed from $\bm{X}^{(l-1)}$ and $\bm{T}_i^{(l-1)}$, for $i\in[1,n]$, respectively, i.e., the output of the previous layer. This clear distinction was not made in HIMP~\cite{himp} and derived works~\cite{wollschlager_expressivity_2024}, where embeddings are mapped from the molecular graph to the junction tree, passed within the junction tree, and then mapped back to the molecular graph in the same iteration. An advantage of our method is that the receptive field of each node expands uniformly and predictably across layers.

Finally, we define our $\Readout$ function as the learnable combination of mean pooling applied to the node embeddings of the individual graph representations,  according to
\begin{equation*}
    \bm{h}_G = \frac{1}{|V(G)|}\sum_{i=1}^{|V(G)|} \bm{x}_i^{(L)} \bm{W}_0~\bigoplus_{j=1}^n~\frac{1}{|V(T_j)|}\sum_{i=1}^{|V(T_j)|} \bm{t}_{j,i}^{(L)}\bm{W}_j.
\end{equation*}
Here, $\bigoplus$ denotes a graph level aggregation (e.g., concatenation or summation) and $\bm{W}_0, \bm{W}_1, \dots, \bm{W}_n \in \mathbb{R}^{d \times d}$ are trainable matrices. The vector $\bm{x}_i^{(L)} \in \mathbb{R}^{d}$ denotes the final embedding of the $i$-th node of $G$ and, similarly, $\bm{t}_{j, i}^{(L)} \in \mathbb{R}^{d}$ is the final embedding of the $i$-th node of the $j$-th reduced graph.

\begin{table*}[t]
    \caption{ADMET, Potency, and MoleculeNet results. Cells show 10-run mean/std of test MAE based on the hyperparameters that resulted in the lowest validation score. GNN abstracts GCN, GIN, GAT, GraphSAGE (i.e., best-performing GNN chosen as representative; fine-grained results in Appendix~\ref{apx:extended}). \best{Dark red bold} = best; \secondbest{dark blue bold} = second-best (ties: all best are marked).}
    \label{tab:potency-admet-moleculenet}
    
    \setlength{\tabcolsep}{4pt} %
    \renewcommand{\arraystretch}{1.2}

    \begin{center}
        \begin{small}
            \begin{tabular}{l ccccc cc ccc}
                \toprule
                \textbf{Model} & \multicolumn{5}{c}{\textbf{ADMET} $\downarrow$} & \multicolumn{2}{c}{\textbf{Potency} $\downarrow$} & \multicolumn{3}{c}{\textbf{MoleculeNet} $\downarrow$} \\
                \cmidrule(lr){2-6} \cmidrule(lr){7-8} \cmidrule(lr){9-11}
                & HLM & KSOL & LogD & MDR1 & MLM & MERS & SARS & ESOL & FreeSolv & Lipo \\
                \midrule
                ECFP & 0.56\err{0.02} & 0.42\err{0.01} & 0.86\err{0.02} & 0.38\err{0.01} & \secondbest{0.55}\err{0.02} & 0.79\err{0.01} & 0.57\err{0.01} & 1.21\err{0.06} & 2.94\err{0.06} & 0.73\err{0.02} \\
                \rowcolor{gray!10} 
                GNN  & 0.56\err{0.03} & 0.48\err{0.07} & \best{0.68}\err{0.07} & 0.37\err{0.02} & 0.64\err{0.05} & 0.71\err{0.02} & 0.45\err{0.05} & \best{0.71}\err{0.04} & \best{1.58}\err{0.17} & \best{0.52}\err{0.02} \\
                HIMP & \secondbest{0.54}\err{0.05} & \best{0.35}\err{0.04} & 0.80\err{0.05} & \secondbest{0.35}\err{0.03} & 0.56\err{0.06} & \best{0.64}\err{0.03} & \best{0.39}\err{0.04} & \secondbest{0.80}\err{0.07} & \secondbest{1.77}\err{0.15} & \best{0.52}\err{0.02} \\
                \rowcolor{gray!10} 
                XIMP & \best{0.53}\err{0.08} & \secondbest{0.37}\err{0.04} & \secondbest{0.69}\err{0.03} & \best{0.31}\err{0.03} & \best{0.49}\err{0.02} & \secondbest{0.69}\err{0.05} & \secondbest{0.41}\err{0.03} & 0.82\err{0.09} & 1.83\err{0.17} & \best{0.52}\err{0.02} \\
                \bottomrule
            \end{tabular}
        \end{small}
    \end{center}
\end{table*}

\begin{table*}[t]
    \caption{ADMET, Potency, and MoleculeNet results. Cells show 10-run mean/std test MAE for the best hyperparameters for each model chosen on the given test dataset. GNN abstracts GCN, GIN, GAT, GraphSAGE (i.e., best-performing GNN model shown; fine-grained results in Appendix~\ref{apx:extended}). \best{Dark red bold} = best; \secondbest{dark blue bold} = second-best (ties: all best are marked).}
    \label{tab:potency-admet-moleculenet-extended}
    
    \setlength{\tabcolsep}{4pt} %
    \renewcommand{\arraystretch}{1.2}

    \begin{center}
        \begin{small}
            \begin{tabular}{l ccccc cc ccc}
                \toprule
                \textbf{Model} & \multicolumn{5}{c}{\textbf{ADMET} $\downarrow$} & \multicolumn{2}{c}{\textbf{Potency} $\downarrow$} & \multicolumn{3}{c}{\textbf{MoleculeNet} $\downarrow$} \\
                \cmidrule(lr){2-6} \cmidrule(lr){7-8} \cmidrule(lr){9-11}
                & HLM & KSOL & LogD & MDR1 & MLM & MERS & SARS & ESOL & FreeSolv & Lipo \\
                \midrule
                ECFP & \best{0.48}\err{0.02} & 0.38\err{0.02} & 0.72\err{0.01} & 0.36\err{0.02} & \secondbest{0.54}\err{0.02} & 0.75\err{0.01} & 0.52\err{0.02} & 1.16\err{0.05} & 2.83\err{0.07} & 0.73\err{0.02} \\
                \rowcolor{gray!10} 
                GNN  & 0.54\err{0.05} & 0.45\err{0.08} & \best{0.67}\err{0.03} & 0.39\err{0.03} & 0.56\err{0.05} & 0.71\err{0.02} & 0.42\err{0.02} & \best{0.71}\err{0.04} & \best{1.62}\err{0.14} & \best{0.52}\err{0.02} \\
                HIMP & \secondbest{0.49}\err{0.04} & \secondbest{0.34}\err{0.06} & 0.81\err{0.07} & \secondbest{0.33}\err{0.03} & 0.57\err{0.06} & \best{0.64}\err{0.03} & \secondbest{0.41}\err{0.06} & \secondbest{0.79}\err{0.07} & 1.82err{0.10} & \secondbest{0.54}\err{0.02} \\
                \rowcolor{gray!10} 
                XIMP & \best{0.48}\err{0.07} & \best{0.33}\err{0.03} & \secondbest{0.69}\err{0.06} & \best{0.32}\err{0.01} & \best{0.52}\err{0.07} & \secondbest{0.68}\err{0.04} & \best{0.38}\err{0.03} & 0.83\err{0.08} & \secondbest{1.77}\err{0.16} & \best{0.52}\err{0.01} \\
                \bottomrule
            \end{tabular}
        \end{small}
    \end{center}
\end{table*}

\subsection{Expressivity and Complexity}
While HIMP couples the molecular graph with a single junction-tree abstraction, XIMP supports multiple reduced graphs with both indirect and direct inter-message passing, strictly subsuming HIMP. XIMP admits richer representations via cross-abstraction embeddings, structured chemical priors, and multi-resolution coarsenings. We formalize this relation between HIMP and XIMP's expressivity as follows (proofs in Appendix~\ref{apx:experessivity}).

\begin{theorem} \label{thm:mono}
    Let the hypothesis classes realized by HIMP and XIMP with depth $L$, hidden dimension $d$, and number of abstractions $n$ be denoted by $\mathcal{H}_{\mathrm{HIMP}}(L,d)$ and $\mathcal{H}_{\mathrm{XIMP}}(L,d,n)$. Then for any $L,d$ and $n \geq 1$,
    \(
    \mathcal{H}_{\mathrm{HIMP}}(L,d) \;\subseteq\; \mathcal{H}_{\mathrm{XIMP}}(L,d,n).
    \)
\end{theorem}

By construction, XIMP inherits and extends HIMP’s ability to exceed 1-WL expressivity (see Section~\ref{sec:bg}).
As some approaches~\citep{ergsimilar,allreduced, kengkanna2024enhancing} use multiple graph representations of small molecules by combining their final graph embeddings, we investigate the specific impact of inter-message passing on expressivity. To this end, we construct a unified graph comprising the molecular graph and its reduced graphs, augmented with inter-graph edges, that captures message passing in XIMP.
Formally, the \emph{compound graph} $\mathcal{G}(G; T_1, T_2,\dots, T_n)=(\mathcal{V}, \mathcal{E})$ of a molecule is defined by
\begin{align*}
  \mathcal{V} & = V(G) \uplus \biguplus_{i=1}^n V(T_i) \\
  \mathcal{E} & = E(G) \uplus \biguplus_{i=1}^n E(T_i) \uplus E_{\mathrm{X}} \\
  E_{\mathrm{X}} &= \{(v,u) \in \mathcal{V} \times \mathcal{V} \mid \exists i \in [1,n]\colon \bm{S}_i[v,u]=1\}.
\end{align*}
The message passing in XIMP using MP and I\textsuperscript{2}MP follows the structure of this compound graph.

\begin{proposition} \label{prop:stru}
    There exist two molecules whose molecular graphs $G$ and $G'$, junction trees $T$ and $T'$, and extended reduced graphs $R$ and $R'$ are each indistinguishable by unlabeled $1$-WL, while $\mathcal{G}(G; T, R)$ is distinguishable from $\mathcal{G}(G'; T', R')$  by unlabeled $1$-WL.
\end{proposition}
Proposition~\ref{prop:stru} suggests XIMP’s expressivity benefits potentially transfer beyond molecular datasets, where abstractions or entire graphs may be unlabeled.

Regarding complexity, XIMP’s per-layer runtime scales linearly with molecular graph size but adds a quadratic cost in hidden dimension and number of abstractions due to inter-message passing. Parameter count rises from HIMP’s single quadratic term to additional linear and quadratic contributions, though modest in practice with few abstractions. Memory is dominated by node embeddings, with extra overhead from mapping matrices that also scale linearly and quadratically. Overall, XIMP maintains linear scaling with graph size and quadratic scaling with the number of abstractions, which we consider a reasonable trade-off given the observed gains in predictive performance.
For a detailed analysis, see Appendix~\ref{apx:complexity}.

\subsection{Assumptions and Limitations} \label{sec:assumptions}
In our cheminformatics setting, we assume that abstractions such as the junction tree and the extended reduced graph provide useful inductive biases, and that the preprocessing (e.g., ring detection, pharmacophore tagging) yields accurate graph abstraction mappings. The complexity of XIMP grows polynomially in the number of abstractions $n$, which we restrict to $n \leq 2$ in practice.

\section{Experimental Evaluation} \label{sec:eval}

\begin{figure*}
    \begin{center}
        \includegraphics[trim=0.25cm 0.25cm 0.25cm 0.25cm, clip,width=\textwidth]{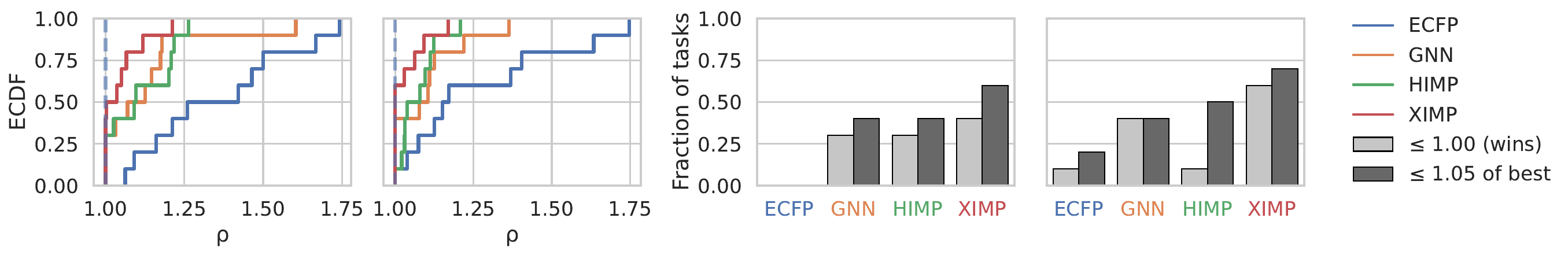}
        \caption{Performance profiles of architectures on ADMET, Potency, and MoleculeNet, computed from Tables~\ref{tab:potency-admet-moleculenet} and~\ref{tab:potency-admet-moleculenet-extended}. Panels 1-2 (left): Empirical Cumulative Distribution Functions (ECDFs) of the performance ratio $\rho=\text{MAE}(\text{model},\text{task})/\min_{\text{arch}}\text{MAE}(\text{task})$, where $\rho=1$ is best and $\rho>1$ quantifies degradation; curves closer to the top-left indicate better overall performance. Panels 3-4 (right): discrete summaries across tasks. Bars show the fraction of tasks a model \emph{wins} ($\rho=1$) or is within a practical tolerance ($\rho\le\tau$, here $\tau=1.05$). Panels 1, 3 use test MAE with hyperparameters chosen by validation; panels 2, 4 use test MAE under best hyperparameters per model and dataset.}
        \label{fig:profiles}
    \end{center}
\end{figure*}

To evaluate the hypothesis that XIMP enhances predictive performance, we benchmarked it against HIMP, several widely used GNN baselines, and Extended Connectivity Fingerprints (ECFP)~\citep{rogers_extended-connectivity_2010}, a standard representation in molecular tasks. We briefly describe our experimental framework
and, to ensure reproducibility, provide details in Appendix~\ref{app:datasets_hardware} and ~\ref{app:hyper_params}.  For robust evidence, we conduct extensive hyperparameter search and ablations, totaling approximately $1$M training runs.

\subsection{Models and Datasets}
Our setup consists of an encoder followed by a regression head. The regression head is implemented as a $k$-layer multilayer perceptron (MLP) with ReLU activations, applied to the graph-level embeddings produced by the encoder. For encoding, we compare GCN~\citep{gcn}, GIN~\citep{gin}, GAT~\citep{velickovic2018graph}, GraphSAGE~\citep{graphsage}, and HIMP (Section~\ref{sec:himp}), plus an ECFP-based non-learnable baseline where only the regression head is trained. Together, these cover convolutional, attention-based, and inductive aggregations. We evaluate on ten prediction tasks from MoleculeNet~\citep{wu2018moleculenet} and the Polaris challenge~\citep{asap_admet_2025, asap_potency_2025}, a recent dataset targeting ADMET (Absorption, Distribution, Metabolism, Excretion, Toxicity) endpoints and drug-candidate potency.

\subsection{Training, Hyperparameter Search} \label{ssec:tr_hyp_ha}
We trained all models with Adam and MAE loss. Hyperparameters were chosen via stratified 10-fold cross-validation on binned regression targets to mitigate target imbalance during selection. We deliberately used target-based rather than purely scaffold-based CV: with small datasets (Appendix~\ref{app:datasets_hardware}), scaffold $k$-fold splits yield imbalanced target distributions and (at high $k$ with few scaffolds) small, uneven folds, making model selection unreliable. Hyperparameter tuning was performed via grid search. For final evaluation, we held out a $10\%$ \emph{scaffold}-split test set to assess chemical-space generalization, retrained the selected configuration on the remaining data, and evaluated with MAE. We chose this protocol because, to our knowledge, no \emph{stratified} scaffold split exists for regression:~\citet{joeres2025data} consider classification only, and~\citet{zhang2025rethinking} control input-graph similarity rather than target stratification. This induces a mismatch---stratified CV for selection vs.\@ scaffold-based testing---that can complicate hyperparameter choice, as validation folds might not match the test distribution; we nonetheless adopt this conservative setup because scaffold splitting yields a structurally distinct and more realistic test set. Accordingly, we report both, mean/std of the test MAE for hyperparameters selected via mean validation MAE across stratified folds (Table~\ref{tab:potency-admet-moleculenet}) for final assessment, and mean/std of the test MAE for optimal test-selected hyperparameters on the scaffold-split holdout (Table~\ref{tab:potency-admet-moleculenet-extended}) as an optimistic upper bound on performance to illustrate what could be achieved with an ideal hyperparameter search.

\begin{figure*}[t]
    \centering
    \includegraphics[trim=0.05cm 0.2cm 0.1cm 0.2cm, clip, width=1.0\textwidth]{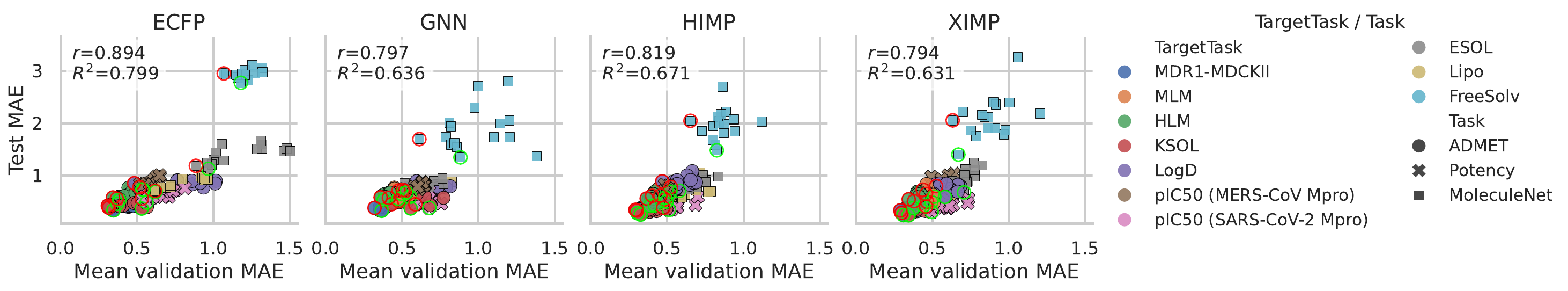}
    \caption{Validation MAE (x-axis) versus test MAE (y-axis) across hyperparameter configurations. For each task, the lowest test MAE is marked by a green circle and the lowest validation MAE by a red circle. Each panel reports the Pearson correlation ($r$) and $R^2$ score. Points closer to the lower left indicate better performance. To improve readability, non-optimal runs (neither red nor green) were randomly subsampled at $150$ per architecture.}
    \label{fig:main_result_2}
\end{figure*}

\begin{figure*}[t]
    \centering
    \includegraphics[trim=0.25cm 0.2cm 0.3cm 0.2cm, clip, width=1.0\textwidth]{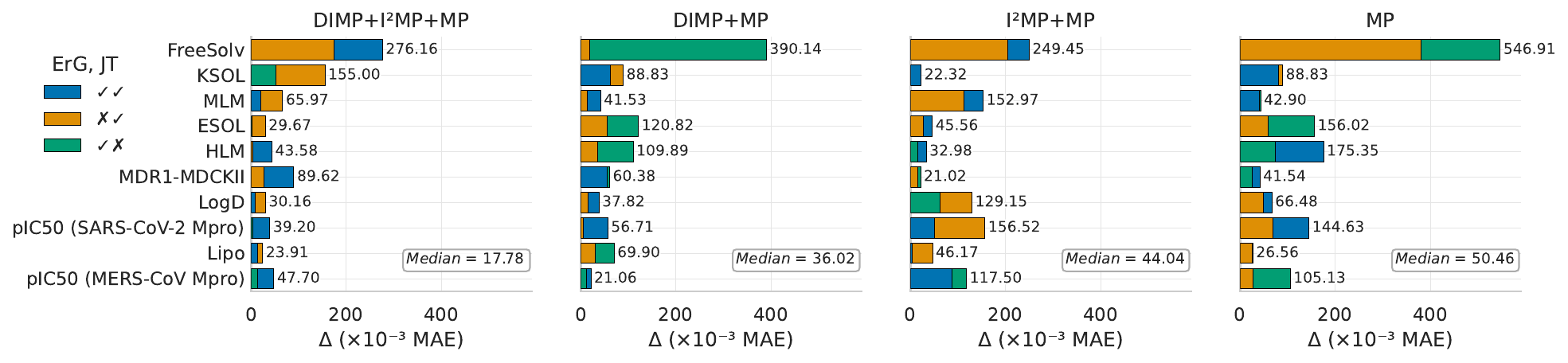}
    \caption{ Relative test error across ErG/JT settings per architecture with medians across targets. Each panel shows $\Delta$ in milli-MAE by target task (x-axis), colored by whether ErG, JT or a combination is used. Here $\Delta = 1000 \times (\mathrm{MAE} - \min_{\text{cfg}}\mathrm{MAE})$, so the best setting is $0$; the other bars report the gap. Bar labels report exact $\Delta$ values, and a horizontal line at $0$ marks the per-target optimum.}
    \label{fig:main_ablation_c}
\end{figure*}

\subsection{Results}
We find that XIMP yields the best predictive performance for the hyperparameters chosen via mean validation MAE in $4/10$ cases, whereas the strongest baseline method (HIMP) likewise outperforms it's competitors on a given dataset in $4/10$ cases (Table~\ref{tab:potency-admet-moleculenet}). We note that choosing our model according to this criterion yields less than optimal results, see Figure~\ref{fig:main_result_2}. This stems from a mismatch between the assumed distributions of target
and scaffolds in the training and test data.

Generally, though, we find that XIMP is the most robust of the evaluated methods. This is reflected in the aggregate view we provide in Figure~\ref{fig:profiles} (left), which shows the per-architecture ECDFs of the results of XIMP and its baselines and is equivalent to the performance profiles introduced by~\citet{dolan2002benchmarking}, a standard method for benchmarking algorithms. XIMP's performance profile dominates all baselines, both for validation-selected and globally best hyperparameters, underscoring its robustness across diverse tasks. Figure~\ref{fig:profiles} (right), which provides a discrete summary perspective across tasks, further underlines XIMP's robustness: more often than any of the baselines, XIMP is either the best performing model measured by its performance ratio $\rho$ or within a practical tolerance of the best performing model. 

On a more granular level, Tables~\ref{tab:potency-admet-moleculenet} and~\ref{tab:potency-admet-moleculenet-extended} both further confirm the strength of XIMP's ability to integrate multiple chemically meaningful and interpretable reduced graphs (junction trees for coarse-grained connectivity, ErGs for fine-grained and pharmacophoric patterns) with the molecular graph itself. This gives it an inductive bias toward functional groups, pharmacophores, and scaffold structures, and, unlike other methods, XIMP also allows for direct message passing between these abstractions, allowing new features to be learned in ways not previously possible. Hence, as expected, for the ADMET tasks that leverage pharmacophoric features (such as HLM, MDR1-MDCKII, MLM), XIMP typically performs better than its competitors, due to its ability to include chemically relevant features (H-Acceptor, H-Donor). 
Conversely, for tasks that rely on the overall structure of the molecule (ESOL, FreeSolv, Lipophilicity), conventional GNNs perform competitively.

In low-data regimes, XIMP's inductive bias also provides a clear advantage. For example, the ADMET dataset comprises only $560$ molecules, in contrast to the substantially larger Potency and MoleculeNet datasets (see Appendix~\ref{app:datasets_hardware}). Under these conditions, XIMP achieves superior performance in $3/5$ or $4/5$ tasks respectively, outperforming competing methods (Tables~\ref{tab:potency-admet-moleculenet} and~\ref{tab:potency-admet-moleculenet-extended}).

\subsection{Hyperparameters and Ablation} \label{sec:extended-results}
We also observe that there is a large discrepancy between test MAE for the hyperparameters chosen via mean validation MAE and globally best test MAE (Figure~\ref{fig:main_result_2}). For the globally best test MAE, we find that XIMP outperforms its baselines in $6/10$ cases and specifically on $4/5$ ADMET tasks, whereas the best-performing baseline methods (GNNs) only accomplish this in $4/10$ cases (the best-performing specific GNN is GIN, which is best in $2/10$ cases, see Appendix~\ref{apx:extended}). While evaluating hyperparameters optimally chosen on the test set naturally constitutes data leakage, we include these results as an optimistic upper bound to performance, indicating the possibility for improvement under perfect tuning.

Finally, we evaluated how well the components of XIMP exploit multiple graph abstractions (i.e., ERG and JT). DIMP+MP is typically the best performing configuration for multi-abstraction settings where ERG and JT are both active (Figure~\ref{fig:main_ablation_c}). In settings where I\textsuperscript{2}MP+MP or I\textsuperscript{2}MP+DIMP+MP are used, XIMP shows a preference for the ERG abstraction. For the configuration without DIMP or I\textsuperscript{2}MP, no clear trend emerges; the utility of abstraction combinations appears largely dataset- and task-dependent. The median relative test error across targets indicates that I\textsuperscript{2}MP+DIMP is most robust to abstraction choice, whereas MP alone is least, underscoring XIMP’s adaptability to diverse abstraction combinations.
A more detailed ablation study is provided in Appendix~\ref{apx:ablation}.

\section{Conclusion} \label{sec:conclusion}
We presented XIMP, a versatile inter-message passing framework that learns across multiple graph abstractions within a single model. XIMP enables both indirect (graph $\leftrightarrow$ abstractions) and direct (abstraction $\leftrightarrow$ abstraction) communication, a learned multi-view readout, and shows potential to mitigate oversquashing and improve long-range communication. Although domain-agnostic, we instantiate it in a chemistry setting with junction trees and extended reduced graphs to demonstrate how interpretable abstractions can be exploited. Across ten diverse property-prediction tasks, XIMP matches or surpasses strong GNN baselines and fixed fingerprints, with gains attributable to explicit cross-abstraction messaging. Our results position multi-abstraction message passing as a principled approach for data-scarce regimes and interpretable graph learning.

Future work includes extending XIMP to additional reductions and attribution tools for tracing predictions across abstractions and extending XIMP to protein design via multi-level graph reductions, e.g., atoms, residues, primary/secondary/tertiary structures. 

\bibliography{ximp}
\bibliographystyle{icml2026}

\newpage
\appendix
\onecolumn

\section{Datasets \& Hardware} \label{app:datasets_hardware}
\paragraph{Datasets}
We evaluated our models on molecular property prediction tasks from two dataset collections: MoleculeNet~\citep{wu2018moleculenet} and Polaris~\citep{asap_admet_2025, asap_potency_2025}. From MoleculeNet, we selected three standard regression benchmarks: solubility (ESOL), hydration free energy (FreeSolv), and lipophilicity (Lipo).
From the Polaris ADMET dataset, we considered four targets relevant to drug absorption and distribution: human liver microsomal stability (HLM), mouse liver microsomal stability (MLM), solubility (KSOL), membrane permeability (MDR1-MDCKII), and lipophilicity (LogD). In addition, we included two potency prediction tasks from Polaris, specifically targeting drug candidate activity against the main proteases of MERS-CoV (MERS-CoV-Mpro) and SARS-CoV-2 (SARS-CoV-2-Mpro), using the pIC50 endpoint. Please find dataset statistics in Table~\ref{tab:datasets} and hardware specifications in Table~\ref{tab:cluster-specs}.

\begin{table}[t]
    \caption{Specifications of GPU compute nodes (node names omitted).}
    \label{tab:cluster-specs}
    \begin{center}
        \begin{small}
            \begin{tabular}{l l l}
                \toprule
                \textbf{GPUs} & \textbf{CPU} & \textbf{Memory} \\
                \midrule
                8$\times$ H100 & Xeon Platinum 8480C (224 thr) & GPU 640\,GB; RAM $\sim$2\,TB \\
                8$\times$ V100 16\,GB & Xeon E5-2698 (80 thr) & GPU 128\,GB; RAM $\sim$500\,GB \\
                4$\times$ H100 & Xeon Platinum 8468V (192 thr) & GPU $\sim$374\,GiB; RAM $\sim$1\,TB \\
                \bottomrule
            \end{tabular}
        \end{small}
    \end{center}
\end{table}

\begin{table}[t]
    \caption{Summary of datasets MoleculeNet, Polaris ADMET, and Polaris Potency used for molecular property and potency prediction.}
    \label{tab:datasets}
    \begin{center}
        \begin{small}
            \begin{tabular}{l l l r}
                \toprule
                \textbf{Dataset Source} & \textbf{Property} & \textbf{Endpoint} & \textbf{Molecules} \\
                \midrule
                MoleculeNet & Solubility & ESOL & 1,128 \\
                 & Hydration Free Energy & FreeSolv & 642 \\
                 & Lipophilicity & Lipo & 4,200 \\
                \midrule
                Polaris ADMET & Metabolic Stability (Human) & HLM & 560 \\
                 & Metabolic Stability (Mouse) & MLM & 560 \\
                 & Solubility & KSOL & 560 \\
                 & Distribution Coefficient & LogD & 560 \\
                 & Membrane Permeability & MDR1-MDCKII & 560 \\
                \midrule
                Polaris Potency & MERS-CoV Mpro & pIC50 & 1,328 \\
                 & SARS-CoV-2 Mpro & pIC50 & 1,328 \\
                \bottomrule
            \end{tabular}
        \end{small}
    \end{center}
\end{table}

\section{Reduced Graphs Construction} \label{sec:reduced_detail}
\paragraph{Junction Trees}
To construct a \textit{junction tree}, the process begins by converting the molecular structure into a graph. This graph is then partitioned into substructures by identifying all simple rings using RDKit’s \texttt{GetSymmSSSR} function~\citep{rdkit}. Next, all edges not belonging to any cycle are identified. Each simple ring and each such edge is treated as a separate cluster, represented as a node in the cluster graph, where a cluster consists of the set of atoms in the corresponding ring, or the two atoms connected by the corresponding edge. Edges are added between clusters if they share at least one atom.

The resulting cluster graph may itself contain cycles, which can lead to non-unique mappings from the molecular graph to the cluster graph. This can occur, for example, when an atom has three or more substituents. To prevent this, the intersecting atom is added as its own cluster, and the bonds responsible for the cycle are removed. Eliminating these cycles ensures an injective mapping from the molecular graph to the cluster graph and removes ambiguities in the decomposition, yielding a unique and well-defined representation. After cycle removal, the junction tree is obtained. The full process is illustrated in Figure~\ref{fig:junction_tree}.

For each node in the junction tree, the following categories are assigned as one-hot encodings: $\{\textrm{singleton}, \textrm{bond}, \textrm{ring}, \textrm{bridged compound} \}$. The edges do not contain any feature information. From a chemical perspective, this decomposition facilitates information flow across multiple fused rings in graph neural networks. For instance, in conventional message passing, the top left carbon atom in the aromatic ring shown in Figure~\ref{fig:junction_tree} could not exchange information with the chlorine atom because they are five hops apart. In the junction tree, however, they are only two hops apart. This condensation of cyclic structures enables the model to capture longer-range interactions more efficiently.

\paragraph{Extended reduced Graphs.}
An \emph{extended reduced graph} (ErG)~\citep{erg} is a simplified molecular representation that captures pharmacophoric and interaction-relevant features while abstracting unnecessary atomic detail. It builds on reduced graphs, which represent chemically meaningful moieties (e.g., H-bond donors/acceptors, aromatic rings) as pseudo-nodes linked according to relationships in the parent molecule. ErGs extend this by treating ring systems separately from hydrogen-bonding and charge features—where prior methods often folded these into the ring node—improving discrimination between compounds with different molecular skeletons. They also fully enumerate “flip-flop” atoms that can act as either H-bond donors or acceptors.

The process of generating an ErG from a chemical structure, as well as the node mapping is shown described below and shown in Figure~\ref{fig:erg}.

\begin{enumerate}
    \item Atoms are \textbf{charged} to represent the molecule under physiological conditions.
    \item \textbf{Initial Atom Property Assignment}: H-bond donor and H-bond acceptor properties are assigned to the atoms. Atoms that can function as both H-bond donor and H-bond acceptor receive a distinct \textit{flip-flop} property.
    \item \textbf{Endcap Group Identification}: These are lateral hydrophobic features, typically composed of three atoms. Thioethers are also identified as endcap groups, and are assigned a property.
    \item \textbf{Ring System Abstraction}: Rings are abstracted as follows to capture their overall properties:
        \begin{enumerate}
            \item Add a centroid atom for each ring and assign it a feature (\textit{aromatic, hydrophobic}).
            \item Retain all substituted ring atoms and create bonds from the centroid to that atom.
            \item Retain all bridgehead atoms (atoms belonging to two or more rings), and create bonds from the centroids to those atoms.
            \item Remove all non-substituted ring atoms and retain all bonds between the atoms that were retained in the previous two steps.
        \end{enumerate}
\end{enumerate}

During the ring abstraction process, if any atom constituting the ring is assigned a specific property, that atom is considered a connected node within the ring framework. The ErG encodes the following node features: \textit{H-bond donor, H-bond acceptor, positive charge, negative charge, hydrophic, aromatic}.

\section{Hyperparameter Search Space} \label{app:hyper_params}
For the GNN-based methods, we explored the following hyperparameter space: number of message-passing layers (1, 2, or 3), hidden dimensions (16 or 32), output embedding dimension (16 or 32), batch size (64 or 128), hidden dimension of the regression head (16 or 32), and number of training epochs (50, 100, or 150). The learning rate, Adam weight decay, dropout rate, and the number of bins for regression stratification remained fixed at $10^{-3}$, $10^{-4}$, $0.1$, and 10, respectively.

For XIMP specifically, we included additional hyperparameters controlling the selection of graph abstraction schemes (ERG vs. JT), the choice of message-passing schemes (I\textsuperscript{2}MP vs. DIMP (Section~\ref{sec:ximp}), binary hyperparameters), and the granularity (resolution) of JT, varied as an integer from 1 to 3. Moreover, we included the selection of the embedding dimension of the reduced graphs (16 and 32). For HIMP, we included a binary hyperparameter choosing whether inter-message passing is active or not.

For ECFP, we examined the daylight atomic invariants initial atom identifiers with output channels (16, 32, 1024, 2048). We decided to include the fingerprints of size 16 and 32 for completeness and to allow for a direct comparison with GNN model embedding dimensions. The fingerprint radius was set to values (2, 3, or 4) resulting in ECFP\_4, ECFP\_6, and ECFP\_8 variants.

\section{Extended Results} \label{apx:extended}
In Tables~\ref{tab:potency-admet-moleculenet_long} and~\ref{tab:potency-admet-moleculenet-extended_long}, we provide a fine-grained view of our results that more clearly highlights the limitations of conventional GNNs than the aggregated summary in the main paper. Across both tables, conventional architectures (GIN, GCN, GAT, GraphSAGE) are the top performers on at most 2/10 tasks, whereas XIMP leads on 4/10 tasks with validation-selected hyperparameters and on 6/10 tasks with per-dataset best hyperparameters (evaluated on the test set). This underscores the benefit of advanced inter-message-passing methods that integrate chemically meaningful, interpretable graph abstractions: XIMP, in particular, shows a clear advantage over competing architectures, especially under the best-hyperparameter setting.

\begin{table*}[t]
    \caption{Potency, ADMET, and MoleculeNet results. Cells show the test MAE based on the hyperparameters that resulted in the lowest validation score. The best result for a given task is marked in bold.}
    \label{tab:potency-admet-moleculenet_long}
    
    \setlength{\tabcolsep}{4pt} 
    \renewcommand{\arraystretch}{1.2}

    \begin{center}
        \begin{small}
            \begin{tabular}{l ccccc cc ccc}
                \toprule
                \textbf{Model} & \multicolumn{5}{c}{\textbf{ADMET} $\downarrow$} & \multicolumn{2}{c}{\textbf{Potency} $\downarrow$} & \multicolumn{3}{c}{\textbf{MoleculeNet} $\downarrow$} \\
                \cmidrule(lr){2-6} \cmidrule(lr){7-8} \cmidrule(lr){9-11}
                & HLM & KSOL & LogD & MDR1 & MLM & MERS & SARS & ESOL & FreeSolv & Lipo \\
                \midrule
                ECFP & 0.56\err{0.02} & 0.42\err{0.01} & 0.86\err{0.02} & 0.38\err{0.01} & 0.55\err{0.02} & 0.79\err{0.01} & 0.57\err{0.01} & 1.21\err{0.06} & 2.94\err{0.06} & 0.73\err{0.02} \\
                \rowcolor{gray!10} 
                GAT & 0.62\err{0.03} & 0.52\err{0.05} & 0.75\err{0.08} & 0.41\err{0.02} & 0.66\err{0.05} & 0.71\err{0.02} & 0.45\err{0.05} & 0.74\err{0.04} & 2.02\err{0.32} & 0.61\err{0.02} \\
                GCN & 0.57\err{0.05} & 0.48\err{0.07} & \textbf{0.68}\err{0.07} & 0.40\err{0.04} & 0.66\err{0.04} & 0.71\err{0.02} & 0.48\err{0.05} & 0.74\err{0.02} & 1.75\err{0.25} & 0.59\err{0.02} \\
                \rowcolor{gray!10} 
                GIN & 0.56\err{0.02} & 0.50\err{0.10} & 0.75\err{0.06} & 0.41\err{0.06} & 0.64\err{0.05} & 0.73\err{0.02} & 0.46\err{0.04} & \textbf{0.71}\err{0.04} & 1.91\err{0.27} & \textbf{0.52}\err{0.02} \\
                GraphSAGE & 0.56\err{0.03} & 0.49\err{0.07} & 0.69\err{0.06} & 0.37\err{0.02} & 0.64\err{0.05} & 0.73\err{0.02} & 0.52\err{0.05} & 0.73\err{0.04} & \textbf{1.58}\err{0.17} & 0.54\err{0.02} \\
                \rowcolor{gray!10} 
                HIMP & 0.54\err{0.05} & \textbf{0.35}\err{0.04} & 0.80\err{0.05} & 0.35\err{0.03} & 0.56\err{0.06} & \textbf{0.64}\err{0.03} & \textbf{0.39}\err{0.04} & 0.80\err{0.07} & 1.77\err{0.15} & \textbf{0.52}\err{0.02} \\
                XIMP & \textbf{0.53}\err{0.08} & 0.37\err{0.04} & 0.69\err{0.03} & \textbf{0.31}\err{0.03} & \textbf{0.49}\err{0.02} & 0.69\err{0.05} & 0.41\err{0.03} & 0.82\err{0.09} & 1.83\err{0.17} & \textbf{0.52}\err{0.02} \\
                \bottomrule
            \end{tabular}
        \end{small}
    \end{center}
\end{table*}

\begin{table*}[t]
    \caption{Potency, ADMET, and MoleculeNet results. Cells show test MAE for the best hyperparameters for each model chosen on the given dataset. The best result for a given task is marked in bold.}
    \label{tab:potency-admet-moleculenet-extended_long}
    
    \setlength{\tabcolsep}{4pt} 
    \renewcommand{\arraystretch}{1.2}

    \begin{center}
        \begin{small}
            \begin{tabular}{l ccccc cc ccc}
                \toprule
                \textbf{Model} & \multicolumn{5}{c}{\textbf{ADMET} $\downarrow$} & \multicolumn{2}{c}{\textbf{Potency} $\downarrow$} & \multicolumn{3}{c}{\textbf{MoleculeNet} $\downarrow$} \\
                \cmidrule(lr){2-6} \cmidrule(lr){7-8} \cmidrule(lr){9-11}
                & HLM & KSOL & LogD & MDR1 & MLM & MERS & SARS & ESOL & FreeSolv & Lipo \\
                \midrule
                ECFP & \textbf{0.48}\err{0.02} & 0.38\err{0.02} & 0.72\err{0.01} & 0.36\err{0.02} & 0.54\err{0.02} & 0.75\err{0.01} & 0.52\err{0.02} & 1.16\err{0.05} & 2.83\err{0.07} & 0.73\err{0.02} \\
                \rowcolor{gray!10} 
                GAT & 0.59\err{0.04} & 0.45\err{0.08} & 0.76\err{0.05} & 0.40\err{0.05} & 0.66\err{0.07} & 0.71\err{0.02} & 0.42\err{0.02} & 0.74\err{0.05} & 1.82\err{0.27} & 0.60\err{0.02} \\
                GCN & 0.54\err{0.05} & 0.50\err{0.07} & 0.69\err{0.05} & 0.39\err{0.05} & 0.64\err{0.06} & 0.71\err{0.02} & 0.45\err{0.03} & 0.75\err{0.02} & \textbf{1.62}\err{0.14} & 0.59\err{0.03} \\
                \rowcolor{gray!10} 
                GIN & 0.54\err{0.03} & 0.50\err{0.07} & 0.70\err{0.04} & 0.40\err{0.06} & 0.56\err{0.05} & 0.74\err{0.01} & 0.49\err{0.07} & \textbf{0.71}\err{0.04} & 2.02\err{0.33} & \textbf{0.52}\err{0.02} \\
                GraphSAGE & 0.56\err{0.04} & 0.52\err{0.08} & \textbf{0.67}\err{0.03} & 0.39\err{0.03} & 0.63\err{0.03} & 0.73\err{0.03} & 0.50\err{0.07} & 0.73\err{0.04} & 2.06\err{0.54} & 0.54\err{0.02} \\
                \rowcolor{gray!10} 
                HIMP & 0.49\err{0.04} & 0.34\err{0.06} & 0.81\err{0.07} & 0.33\err{0.03} & 0.57\err{0.06} & \textbf{0.64}\err{0.03} & 0.41\err{0.06} & 0.79\err{0.07} & 1.82\err{0.10} & 0.54\err{0.02} \\
                XIMP & \textbf{0.48}\err{0.07} & \textbf{0.33}\err{0.03} & 0.69\err{0.06} & \textbf{0.32}\err{0.01} & \textbf{0.52}\err{0.07} & 0.68\err{0.04} & \textbf{0.38}\err{0.03} & 0.83\err{0.08} & 1.77\err{0.16} & \textbf{0.52}\err{0.01} \\
                \bottomrule
            \end{tabular}
        \end{small}
    \end{center}
\end{table*}

\section{Ablation Study} \label{apx:ablation}
We conducted an extensive evaluation study to evaluate the impact of the different message passing schemes (I\textsuperscript{2}MP and DIMP), graph reductions (ERG, JT), and their resolutions. As we deem it the most relevant metric due to the necessity to generalize to chemical scaffold unseen during training, we conduct our ablation study using test MAE of the models trained with the hyperparameters chosen via mean validation MAE.

\paragraph{Message Passing Schemes.}
Concerning message passing schemes, we find that for the test MAE of the models trained with the hyperparameters chosen via mean validation MAE, DIMP yields the best results most reliably. As shown in Table~\ref{tab:ximp-ablation-valsel}, XIMP employing DIMP outperforms other higher order message passing schemes in only $5$ out of $10$ cases, which is seconded by XIMP with DIMP+I\textsuperscript{2}MP. 
Our results together illustrate the effectiveness of our devised higher order intra message passing schemes and likewise highlight the necessity of careful hyperparameter selection.

\begin{table}[t]
    \caption{XIMP ablation (a/b/c/d) with XIMP (a) (I\textsuperscript{2}MP+MP), XIMP (b) (DIMP+MP), XIMP (c) (MP; plain message passing with late fusion of the different graph-level embedings) across all target tasks (columns), and XIMP (d) (DIMP+I\textsuperscript{2}MP+MP). Test MAE of validation‑selected configs only. The best result for a task is marked in bold.}
    \label{tab:ximp-ablation-valsel}
    \begin{center}
        \footnotesize
        \begin{tabular}{l c c c c c c c c c c}
            \toprule
            \textbf{Model} & \multicolumn{5}{c}{\textbf{ADMET} $\downarrow$} & \multicolumn{2}{c}{\textbf{Potency} $\downarrow$} & \multicolumn{3}{c}{\textbf{MoleculeNet} $\downarrow$} \\
            \cmidrule(lr){2-6} \cmidrule(lr){7-8} \cmidrule(lr){9-11}
            & HLM & KSOL & LogD & \begin{tabular}[c]{@{}c@{}}MDR1-\\MDCKII\end{tabular} & MLM & \begin{tabular}[c]{@{}c@{}}pIC50\\MERS\end{tabular} & \begin{tabular}[c]{@{}c@{}}pIC50\\SARS\end{tabular} & ESOL & FreeSolv & Lipo \\
            \midrule
            XIMP (a) & 0.543 & 0.382 & 0.731 & \textbf{0.319} & 0.625 & \textbf{0.644} & 0.404 & 0.794 & 1.907 & 0.526 \\
            \rowcolor{gray!10} XIMP (b) & \textbf{0.486} & \textbf{0.280} & \textbf{0.726} & 0.365 & 0.507 & \textbf{0.644} & \textbf{0.371} & 0.794 & 2.054 & \textbf{0.494} \\
            XIMP (c) & \textbf{0.486} & \textbf{0.280} & 0.754 & 0.330 & \textbf{0.493} & 0.657 & 0.515 & 0.818 & 2.054 & 0.512 \\
            \rowcolor{gray!10} XIMP (d) & 0.569 & 0.297 & 0.801 & 0.376 & \textbf{0.493} & 0.761 & 0.404 & \textbf{0.765} & \textbf{1.877} & 0.536 \\
            \bottomrule
        \end{tabular}
    \end{center}
\end{table}

\paragraph{Junction Tree Resolution.}
Moreover, we evaluated how XIMP with different combinations of intermessage passing (i.e., I\textsuperscript{2}MP, DIMP) responds to different JT resolutions (Table~\ref{tab:ft-ximp-combined}). To this purpose, we conducted the following experiments only using only JTs and evaluated test MAE for the hyperparameters selected by lowest validation MAE per jt\_resolution. Our results indicate that this is in large parts a datasets and task dependent property. For example, in ADMET's KSOL and LogD tasks, a resolution of $2$ or more appeared to be preferable in most DIMP and I\textsuperscript{2}MP combinations, whereas, for example, for the MDR1-MDCKII and MLM tasks, a resolution of $1$ was preferable in almost all cases. For the pIC50 tasks as well as ESOL and FreeSolv, the architectural combination seemed to have a larger impact; with, for example a resolution of $1$ being optimal for pIC50 (SARS-CoV2 Mpro) and XIMP with I\textsuperscript{2}MP and a resolution of $3$ being optimal for pIC50 (SARS-CoV2 Mpro) and XIMP with DIMP. To conclude, it appears that the optimal combination of message passing schemes and feature tree resolution is not only highly dependent on the underlying data but also the task to be learned; indicating latent mechanisms deeply rooted in the chemical relevance of the abstractions and the corresponding information flow between them.

\begin{table}[t]
    \caption{XIMP ablation (a/b/c/base) with XIMP (a) (I\textsuperscript{2}MP only), XIMP (b) (DIMP only), and XIMP (c) (neither; plain message passing with late fusion of the different graph-level embedings) across all target tasks (columns) showing the test MAE for configurations selected by the lowest validation MAE within each jt\_coarseness (rows) and target task (columns). The table isolates how changing jt\_resolution $\in \{1,2,3\}$ impacts generalization for each XIMP variant across targets. Lower is better. The absolute per-target minima (best raw test MAE) for each variant and task are marked in bold.}
    \label{tab:ft-ximp-combined}
    \begin{center}
        \footnotesize
        \begin{tabular}{l l c c c c c c c c c c}
            \toprule
            \textbf{Model} & \begin{tabular}[c]{@{}c@{}}\textbf{JT}\\\textbf{coar.}\end{tabular} & \multicolumn{5}{c}{\textbf{ADMET} $\downarrow$} & \multicolumn{2}{c}{\textbf{Potency} $\downarrow$} & \multicolumn{3}{c}{\textbf{MoleculeNet} $\downarrow$} \\
            \cmidrule(lr){3-7} \cmidrule(lr){8-9} \cmidrule(lr){10-12}
            & & HLM & KSOL & LogD & \begin{tabular}[c]{@{}c@{}}MDR1-\\MDCKII\end{tabular} & MLM & \begin{tabular}[c]{@{}c@{}}pIC50\\MERS\end{tabular} & \begin{tabular}[c]{@{}c@{}}pIC50\\SARS\end{tabular} & ESOL & FreeSolv & Lipo \\
            \midrule
            XIMP (a) & 1 & \textbf{0.483} & \textbf{0.337} & 0.757 & \textbf{0.298} & \textbf{0.513} & 0.812 & \textbf{0.453} & \textbf{0.812} & 1.952 & 0.526 \\
            \rowcolor{gray!10} XIMP (a) & 2 & 0.543 & 0.387 & 0.731 & 0.315 & 0.625 & \textbf{0.731} & 0.501 & 0.920 & 2.043 & \textbf{0.521} \\
            XIMP (a) & 3 & 0.555 & 0.382 & \textbf{0.730} & 0.359 & 0.529 & 0.795 & 0.542 & 0.890 & \textbf{1.800} & 0.583 \\
            \midrule
            \rowcolor{gray!10} XIMP (b) & 1 & \textbf{0.450} & \textbf{0.343} & 0.726 & \textbf{0.350} & 0.535 & \textbf{0.665} & 0.428 & \textbf{0.794} & \textbf{1.664} & \textbf{0.494} \\
            XIMP (b) & 2 & 0.593 & 0.452 & \textbf{0.698} & 0.360 & \textbf{0.488} & 0.852 & \textbf{0.363} & 0.825 & 1.743 & 0.527 \\
            \rowcolor{gray!10} XIMP (b) & 3 & 0.548 & 0.386 & 0.731 & 0.369 & 0.489 & 0.689 & 0.397 & 0.828 & 2.080 & 0.529 \\
            \midrule
            XIMP (c) & 1 & 0.536 & 0.361 & 0.792 & \textbf{0.315} & \textbf{0.491} & \textbf{0.630} & 0.515 & \textbf{0.759} & 1.890 & \textbf{0.513} \\
            \rowcolor{gray!10} XIMP (c) & 2 & \textbf{0.512} & 0.368 & 0.754 & 0.364 & 0.516 & 0.709 & \textbf{0.403} & 0.816 & \textbf{1.507} & 0.538 \\
            XIMP (c) & 3 & 0.661 & \textbf{0.321} & \textbf{0.683} & 0.346 & 0.639 & 0.707 & 0.434 & 0.889 & 1.796 & 0.516 \\
            \midrule
            \rowcolor{gray!10} XIMP & 1 & \textbf{0.536} & 0.333 & 0.829 & 0.376 & \textbf{0.493} & \textbf{0.796} & 0.440 & 0.765 & 1.979 & 0.536 \\
            XIMP & 2 & 0.569 & \textbf{0.297} & 0.801 & 0.343 & 0.517 & 0.848 & \textbf{0.409} & \textbf{0.749} & \textbf{1.862} & \textbf{0.530} \\
            \rowcolor{gray!10} XIMP & 3 & 0.595 & 0.376 & \textbf{0.728} & \textbf{0.342} & 0.557 & 0.958 & 0.441 & 0.935 & 2.138 & 0.550 \\
            \bottomrule
        \end{tabular}
    \end{center}
\end{table}

\section{Graph Isomorphism Networks} \label{apx:gin}
The GIN update (matrix notation, used henceforth) is

\begin{equation*}
    \bm{X}^{(l+1)} \;=\; \mathrm{MLP}^{(l)}\!\left(\,\big(\bm{A} + (1+\epsilon^{(l)})\,\bm{I}\big)\,\bm{X}^{(l)}\right),
\end{equation*}

where $\bm{A} \in \{0,1\}^{|V(G)| \times |V(G)|}$ denotes the adjacency matrix, $\bm{I} \in \mathbb{R}^{|V(G)| \times |V(G)|}$ the identity, $\bm{X}^{(l)} \in \mathbb{R}^{|V(G)| \times d}$ the node feature matrix at layer $l$, and $\epsilon^{(l)} \in \mathbb{R}$ a learnable scalar. For GIN-E the layer update is given by

\begin{equation*}
\bm{X}^{(l+1)} \;=\; \mathrm{MLP}^{(l)}\!\left(\,(1+\epsilon^{(l)})\,\bm{X}^{(l)} \;+\; \mathcal{M}_{\bm{A}}^{(l)}(\bm{X}^{(l)},\bm{E})\right)
\end{equation*}

with 

\begin{equation*}
\big[\mathcal{M}_{\bm{A}}^{(l)}(\bm{X},\bm{E})\big]_v \;=\; \sum_{u=1}^{n} A_{vu}\;\sigma\!\Big( \bm{x}_u^{(l)} \;+\; \bm{E}_{vu} \Big),
\end{equation*}

where $\bm{E} \in R^{|V(G)| \times |V(G)| \times d_e}$ denotes the tensor of edge feature vectors and $\bm{E}_{vu}$ the edge feature vector for the edge between nodes $u$ and $v$. If $d_e \neq d$, an additional learnable transformation can be applied to $\bm{E}_{vu}$ to transform the edge feature vector to the space of the node embeddings.  

\section{Hierachical Inter-Message Passing (HIMP)} \label{apx:himp}
HIMP performs standard message passing in both the molecular graph $G$ and the JT $T$ extended by \emph{inter-message passing} in each layer.
Let $\bm{X}^{(l)} \in \mathbb{R}^{|V(G)| \times d}$ and $\bm{T}^{(l)} \in \mathbb{R}^{|V(T)| \times d}$ denote matrices storing the node embeddings $G$ and $T$ in layer $l$, respectively, and $\bm{S} \in \{0, 1\}^{|V(G)| \times |V(T)|}$ be the mapping matrix encoding the assignment of nodes of the molecular graph to nodes of the JT. Then the inter-message passing step changes the embedding matrices $\bm{X}^{(l)}$ and $\bm{T}^{(l)}$ according to
\begin{equation*}
    \begin{aligned}
        \bm{X}^{(l)} &\leftarrow \bm{X}^{(l)} + \sigma \left( \bm{S} \bm{T}^{(l)} \bm{W}_1^{(l)} \right) \\
        \bm{T}^{(l)} &\leftarrow \bm{T}^{(l)} + \sigma \left( \bm{S}^{\tp} \bm{X}^{(l+1)} \bm{W}_2^{(l)} \right)
    \end{aligned}
\end{equation*}
where $\sigma$ denotes a non-linearity and the matrices $\bm{W}_1^{(l)}$ and $\bm{W}_2^{(l)} \in \mathbb{R}^{d \times d}$ are trainable parameters specific for layer $l$.

The \Readout function after layer $L$ is realized by
\begin{equation}\label{eq:himp:readout}
    \bm{h}_G = \frac{1}{|V(G)|}\sum_{i=1}^{|V(G)|} \bm{x}_i^{(L)}
    +
    \frac{1}{|V(T)|}\sum_{i=1}^{|V(T)|} \bm{t}_{i}^{(L)},
\end{equation}
where $\bm{x}_i^{(L)} \in \mathbb{R}^{d}$ and $\bm{t}_i^{(L)} \in \mathbb{R}^{d}$ are the final embeddings of the $i$-th node of the graph $G$ and tree $T$, respectively. While the publication~\citep{himp} states that 
\begin{equation*}
    \bm{h}_G = \sum_{i=1}^{|V(G)|} \bm{x}_i^{(L)} \ \bigg\Vert\  \sum_{i=1}^{|V(T)|} \bm{t}_{i}^{(L)}
\end{equation*}
is used, where $\Vert$ denotes concatenation, the authors' implementation instead corresponds to Eq.~\eqref{eq:himp:readout}, employing mean aggregation over the node embeddings of both graphs and summing the results. In our experiments, we used the implementation as provided by the authors.

\section{Complexity Analysis} \label{apx:complexity}
Below, we provide an analysis of the time and space complexity of XIMP.

\paragraph{Per-layer time complexity.}
For HIMP, each message-passing layer consists of (i) intra-graph updates on the molecular graph $G$ with $|V(G)|$ nodes and adjacency $\bm{A}$, and (ii) updates on its junction tree $T$ with $|V(T)|$ nodes. This yields a per-layer cost of $\mathcal{O}(|E(G)|d + |V(G)|d^2 + |V(T)|d^2)$, dominated by aggregation and MLP operations.  XIMP generalizes this to $n$ reduced graphs $T_1, \ldots, T_n$, adding (a) indirect inter-message passing ($G \leftrightarrow T_i$) at cost $\mathcal{O}\!\left(\sum_{i=1}^n |V(G)| d^2\right)$ and (b) direct inter-message passing ($T_i \leftrightarrow T_j$ for $i < j$) at cost  $\mathcal{O}\!\left(\sum_{i<j} |V(T_i)| d^2 + |V(T_j)| d^2\right)$. Thus, per-layer complexity scales as

\begin{equation*}
    \mathcal{O}\!\left(|E(G)| d + (|V(G)| + \textstyle\sum_i |V(T_i)|) d^2 + n|V(G)|d^2 + \sum_{i<j} (|V(T_i)|+|V(T_j)|) d^2 \right),
\end{equation*}

which is polynomial in the number of abstractions $n$. Compared to HIMP, XIMP introduces only moderate quadratic overhead in $d$, while enabling richer cross-abstraction communication.

\paragraph{Parameter count.}
HIMP maintains two GNN encoders (GIN-E on $G$, GIN on $T$) plus linear projections for inter-message passing. The parameter count therefore scales as $\Theta(L \cdot d^2)$, with constants depending on MLP depth. XIMP extends this by (i) duplicating the abstraction encoder for each $T_i$, and (ii) introducing additional projection matrices $\bm{W}_{i,1}, \bm{W}_{i,2} \in \mathbb{R}^{d \times d}$ for indirect inter-message passing and $\bm{W}_{i \to j} \in \mathbb{R}^{d \times d}$ for direct inter-abstraction exchange. The resulting parameter count is 

\begin{equation*}
    \Theta\!\Big(L \cdot \big((1+n)\,d^2 + n \cdot d^2 + n^2 \cdot d^2\big)\Big),
\end{equation*}

dominated by $\mathcal{O}(n^2 d^2)$ for pairwise abstractions. While this quadratic dependence in $n$ is more expensive than HIMP, in practice $n \leq 2$ or $3$ (JT, ErG, coarsened JT), making XIMP only a constant-factor increase.

\paragraph{Memory and storage complexity.}
For both HIMP and XIMP, node embeddings per layer require $\mathcal{O}((|V(G)|+\sum_i |V(T_i)|) d)$ memory, with gradient checkpoints doubling this during backpropagation. HIMP stores one mapping matrix $\bm{S} \in \{0,1\}^{|V(G)| \times |V(T)|}$, whereas XIMP stores multiple $\bm{S}_i$ and pairwise compositions $\tilde{\bm{S}}_{ik}$, yielding additional $\mathcal{O}(n |V(G)| + n^2 |V(G)|)$ storage. Parameter storage follows the counts above, $\mathcal{O}((1+n+n^2)d^2)$, which is negligible compared to activations when $|V(G)| \gg d$. Thus, XIMP scales linearly in graph size but quadratically in the number of abstractions $n$—a reasonable trade-off given the interpretability and predictive gains.

\section{Expressivity of XIMP versus HIMP} \label{apx:experessivity}
The expressive power of message-passing neural networks is, in general, limited by the 1-Weisfeiler-Leman (1-WL) test, which characterizes their ability to distinguish non-isomorphic graphs. HIMP augments this framework by jointly operating on the molecular graph $G$ and its junction-tree abstraction $T$, with information exchange mediated by the assignment matrix $\bm{S} \in \{0,1\}^{|V(G)| \times |V(T)|}$. Inter-message passing (IMP) integrates abstract representations after each layer, thereby enabling the model to separate graph pairs indistinguishable by 1-WL on $G$ but distinguishable on $T$. For instance, molecules such as decalin and bicyclopentyl cannot be separated by 1-WL on the molecular graph but are discriminated by their distinct junction trees~\citep{himp}.

XIMP generalizes this architecture by (i) supporting an arbitrary collection of reduced graphs $\{T_i\}_{i=1}^n$, with indirect inter-message passing (I$^2$MP) between $G$ and each $T_i$, and (ii) introducing direct inter-message passing (DIMP) between abstractions via normalized projections $\tilde{\bm{S}}_{ik} = \bm{D}_{T,i}^{-1} \bm{S}_i^\tp \bm{D}_{G,k}^{-1} \bm{S}_k$. These mechanisms yield pairwise abstraction embeddings that cannot be realized by HIMP, while retaining HIMP as the special case $n=1$ with DIMP disabled.

\begin{theorem}[equiv. Theorem~\ref{thm:mono})]
    Let $\mathcal{H}_{\mathrm{HIMP}}(L,d)$ and $\mathcal{H}_{\mathrm{XIMP}}(L,d,n)$ denote the hypothesis classes realized by HIMP and XIMP with depth $L$, hidden dimension $d$, and number of abstractions $n$. Then for any $L,d$ and $n \geq 1$,
    \begin{equation*}
        \mathcal{H}_{\mathrm{HIMP}}(L,d) \;\subseteq\; \mathcal{H}_{\mathrm{XIMP}}(L,d,n).
    \end{equation*}
\end{theorem}

\begin{proof}
    The inclusion follows by construction. For $n=1$, choosing a single abstraction $T_1$ equal to the junction tree and disabling DIMP (which is equivalent to learning all-zero projection matrices for messages other than those passed between $G$ and $T_1$) reduces XIMP to HIMP. Therefore every function in $\mathcal{H}_{\mathrm{HIMP}}(L,d)$ is realizable in $\mathcal{H}_{\mathrm{XIMP}}(L,d,n)$. For $n>1$, XIMP introduces additional encoders and projection matrices $(\bm{W}_{i,1}^{(l)}, \bm{W}_{i,2}^{(l)}, \bm{W}_{i \to j}^{(l)})$, yielding cross-abstraction feature pathways absent in HIMP. Hence $\mathcal{H}_{\mathrm{XIMP}}(L,d,n)$ strictly contains $\mathcal{H}_{\mathrm{HIMP}}(L,d)$ whenever multiple abstractions are employed.
\end{proof}

While neither HIMP nor XIMP extend beyond the formal limitations of $k$-WL in the classical sense, XIMP admits strictly richer hypothesis classes in practice due to: (i) the integration of chemically structured priors across multiple abstractions (junction trees, pharmacophoric ErGs, and multi-resolution variants), (ii) the construction of cross-view embeddings via $\tilde{\bm{S}}_{ik}$, and (iii) the mitigation of oversquashing through multi-coarseness junction trees. Together, these components enlarge the set of practically distinguishable molecular graphs, while preserving HIMP as a special case.

\paragraph{Structural Abstractions Alone Enable Separation.}
To highlight how jointly leveraging multiple graph abstractions can overcome the limitations of 1-WL, we prove the following proposition. It formalizes, in our chemical setting, the I\textsuperscript{2}MP+MP communication pattern employed by XIMP.

\begin{proposition}[equiv. Proposition~\ref{prop:stru}]
    There exist two molecules whose molecular graphs $G$ and $G'$, junction trees $T$ and $T'$, and extended reduced graphs $R$ and $R'$ are each indistinguishable by unlabeled $1$-WL, while $\mathcal{G}(G; T, R)$ is distinguishable from $\mathcal{G}(G'; T', R')$  by unlabeled $1$-WL.
\end{proposition}

\begin{proof}%
    Consider $\mathrm{M}_1=\textit{3-Hydroxypyridine}$ (SMILES \texttt{Oc1cnccc1}) and $\mathrm{M}_2=\textit{4-Hydroxypyridine}$ (SMILES \texttt{Oc1ccncc1}). We run color refinement with constant initialization (no node or edge attributes) in all views.
    
    \emph{Single views.} 
    (i) As the molecular graphs $G_1,G_2$ are each a six-cycle with a single pendant leaf (the -OH group), unlabeled $1$-WL produces identical stable partitions.
    (ii) The junction trees $T_1,T_2$ each consist of one ring cluster (size six) and one bond cluster for the exocyclic O--C bond, joined by one edge, hence unlabeled $1$-WL fails to distinguish the pair.
    (iii) The ErGs $R_1,R_2$ collapse the ring to a centroid and keep O, the substituted ring carbon, and the ring N as nodes, giving in both cases the unlabeled path $O\!-\!C_\ast\!-\!\text{centroid}\!-\!N$ (whereby $C_\ast$ denotes the ring carbon bonded to the exocyclic oxygen). Therefore, they are indistinguishable by unlabeled $1$-WL .
    
    \emph{Compound view.}
    We consider the two compound graphs $\mathcal{G}_i(G_i; T_i, R_i)$, $i \in \{1,2\}$: each ring atom in $G_i$ connects to the ring cluster in $T_i$ and to the ring centroid in $R_i$; the O-bearing carbon $C_\ast$ and O connect to the O--C bond cluster in $T_i$ and to their nodes in $R_i$; the ring N connects to its ErG node and to the centroid. These cross-layer edges \emph{anchor} two specific ring atoms in $G_i$ (the O-bearing carbon $C_\ast$ and the nitrogen $N$) to distinguished endpoints across $T_i$ and $R_i$.
    
    Along the six-membered cycle of $G$, the anchored atoms $C_\ast$ and $N$ are separated by two edges in $\mathrm{M}_1$ and by three edges in $\mathrm{M}_2$. 
    Under unlabeled $1$-WL on $\mathcal{G}_i$, the neighborhood multisets at the anchored nodes (and their witnesses via cross edges into $T_i$ and $R_i$) differ and this asymmetry propagates, yielding distinct stable colorings of $\mathcal{G}_1$ and $\mathcal{G}_2$, even though the pairs $G_i$, $T_i$, and $R_i$ for $i \in \{1,2\}$ are each indistinguishable alone, proving the statement.
\end{proof}

\section{DIMP Normalization} \label{apx:dimpnorm}
Extending upon Section~\ref{sec:dimp}, recall that $\bm D_{G,k}^{-1}$ splits each atom’s contribution evenly across its $T_k$-memberships, preventing over-weighting in the $T_k\!\to\!G$ projection, while $\bm D_{T,i}^{-1}$ then averages these per-atom signals over the atoms summarized by each node of $T_i$. We below shot that this prevents multiplicity-induced overweighting and ensures stable message magnitudes across overlaps.

\begin{proposition}[equiv. Proposition~\ref{prop:dimpnorm}]
    Let $\bm{S}_k \in \{0,1\}^{|V(G)|\times |V(T_k)|}$ and $\bm{S}_i \in \{0,1\}^{|V(G)|\times |V(T_i)|}$ represent left-total relations between the nodes of $G$ and those of $T_k$ and $T_i$, respectively. Let $\widetilde{\bm{S}}_{ik}$ be defined according to Eq.~\eqref{eq:dimp:norm} and $\widetilde{\bm{M}}_{k\rightarrow i}^{(l)} = \widetilde{\bm{S}}_{ik} \bm{T}_k^{(l)}$.
    Then, the following statements hold:
    \begin{enumerate}[noitemsep, topsep=0pt]
        \item There exists $\alpha \in \mathbb{R}$ with $\frac{|V(T_i)|}{|V(T_k)|} \leq \alpha \leq |V(T_i)|$ such that
        $\|\widetilde{\bm{M}}_{k\rightarrow i}^{(l)}\|_\infty \leq \|\bm{T}^{(l)}_k\|_\infty$ and $\|\widetilde{\bm{M}}_{k\rightarrow i}^{(l)}\|_1 \leq \alpha \|\bm{T}^{(l)}_k\|_1$.
        \item For any $\bm{x} \in \mathbb{R}^{d}$ and $\bm{T}_k = \bm{1}\bm{x}^\tp \in \mathbb{R}^{|V(T_k)|\times d}$, it holds $\widetilde{\bm{S}}_{ik}\bm{T}_k = \bm{1}\bm{x}^\tp$.
    \end{enumerate}
\end{proposition}

\begin{proof}
    The proof consists of one fundamental Lemma, from which the claimed statements follow. We begin by showing the Lemma.

    \begin{lemma}[Row-stochasticity]
        \label{lem:rowstoch}
        $\widetilde{\bm{S}}_{ik}$ has nonnegative entries and each row sums to $1$, i.e., $\widetilde{\bm{S}}_{ik}\bm{1}=\bm{1}$.
    \end{lemma}
    
    \begin{proof}
        By construction, $\widetilde{\bm{S}}_k$ is row-normalized: each row $v$ sums to
        $\sum_{u_k} (D_{G,k}^{-1}S_k)_{v,u_k}=%
        1$.
        Similarly, each row $u_i$ of $\widetilde{\bm{S}}_i^\tp$ sums to $\sum_v (D_{T,i}^{-1}S_i^\tp)_{u_i,v}=1$.
        Products of nonnegative row-stochastic matrices are row-stochastic, hence $\widetilde{\bm{S}}_{ik}\bm{1}=\bm{1}$.
    \end{proof}

    \begin{corollary}[Stable row magnitudes]
        \label{cor:nonexpansive1}
        Let \(\|\cdot\|_\infty\) denote the row-wise max norm on matrices, i.e.,
        \(\|\bm{A}\|_\infty=\max_{u}\sum_{j}|A_{u j}|\).
        Then, for any \(\bm{T}_k\in\mathbb{R}^{|V(T_k)|\times d}\),
        
        \begin{equation*}
            \|\widetilde{\bm{S}}_{ik}\,\bm{T}_k\|_\infty \;\le\; \|\bm{T}_k\|_\infty .
        \end{equation*}
        Equivalently, each row \((\widetilde{\bm{S}}_{ik}\,\bm{T}_k)_{u_i,:}\) is a convex combination of the rows of \(\bm{T}_k\).
    \end{corollary}
    
    \begin{proof}
        Each row of \(\widetilde{\bm{S}}_{ik}\) is a probability vector by Lemma~\ref{lem:rowstoch},
        so left-multiplication forms convex combinations of rows of \(\bm{T}_k\).
        The \(\infty\)-operator norm of any row-stochastic matrix equals \(1\),
        hence \(\|\widetilde{\bm{S}}_{ik}\bm{T}_k\|_\infty \le \|\bm{T}_k\|_\infty\).
    \end{proof}    

    \begin{corollary}[Stable column magnitudes] \label{cor:nonexpansive2}
        Let \(\|\cdot\|_1\) denote the column-wise max norm on matrices, i.e.,
        \(\|\bm{A}\|_1=\max_{u}\sum_{j}|A_{u j}|\).
        Then, for any \(\bm{T}_k\in\mathbb{R}^{|V(T_k)|\times d}\), there exists $\alpha \in \mathbb{R}$ s.t. $0 \leq \alpha \leq$ 
        
        \begin{equation*}
        \|\widetilde{\bm{S}}_{ik}\bm{T}_k\|_1 \leq \alpha\| \bm{T}_k\|_1.
        \end{equation*}
        
    \end{corollary}
    
    \begin{proof}
        By the submultiplicative property of induced norms, it follows that $\|\widetilde{\bm{S}}_{ik}\,\bm{T}_k\|_1 \leq \|\widetilde{\bm{S}}_{ik}\|_1\|\bm{T}_k\|_1$. As \(\widetilde{\bm{S}}_{ik} \in \{x \in \mathbb{R} | 0 \leq x \leq 1 \}^{|V(T_i)|\times|V(T_k)|}\) with each row each row of $\widetilde{\bm{S}}_{ik}$ being interpretable as a probability vector by Lemma~\ref{lem:rowstoch}, $\|\widetilde{\bm{S}}_{ik}\|_1 \leq |V(T_i)|$. Hence, $\|\widetilde{\bm{S}}_{ik}\,\bm{T}_k\|_1 \ \leq \alpha\|\bm{T}_k\|_1$ with $\alpha \leq |V(T_i)|$. As the column sum is a max norm, we can further write $\frac{|V(T_i)|}{|V(T_k)|} \leq \alpha$, which represents the average column sum, which is always lower or equal than the max column sum.
    \end{proof}    

    \begin{corollary}[Constant preservation] \label{cor:constant}
        For any \(x\in\mathbb{R}^{d}\) and any constant embedding
        \(\bm{T}_k=\bm{1}\,x^\tp \in \mathbb{R}^{|V(T_k)|\times d}\),
        
        \begin{equation*}
            \widetilde{\bm{S}}_{ik}\,\bm{T}_k \;=\; \bm{1}\,x^\tp .
        \end{equation*}
        
        Equivalently, \(\widetilde{\bm{S}}_{ik}\bm{1}=\bm{1}\).
    \end{corollary}
    
    \begin{proof}
        Using Lemma~\ref{lem:rowstoch}, \(\widetilde{\bm{S}}_{ik}\bm{1}=\bm{1}\).
        Thus \(\widetilde{\bm{S}}_{ik}(\bm{1}x^\tp)=(\widetilde{\bm{S}}_{ik}\bm{1})x^\tp=\bm{1}x^\tp\).
    \end{proof}
    
    In summary, statement 1 follows by Corollaries~\ref{cor:nonexpansive1} and ~\ref{cor:nonexpansive2}, whereby statement 2 is equivalent to~\ref{cor:constant}.
\end{proof}

\end{document}